\documentclass{article}

\usepackage[final,nonatbib]{neurips_2022}

\usepackage[utf8]{inputenc} 
\usepackage[T1]{fontenc}    
\usepackage{url}            
\usepackage{booktabs}       
\usepackage{nicefrac}       
\usepackage{microtype}      
\usepackage[dvipsnames]{xcolor}         
\usepackage{todonotes}

\usepackage{array}
\usepackage{microtype}
\usepackage{graphicx}
\usepackage{booktabs}
\usepackage{multirow, rotating}
\usepackage[flushleft]{threeparttable}
\usepackage{pifont}

\usepackage{amsthm}
\usepackage{amsmath}
\usepackage{amsfonts}
\usepackage{amssymb}
\usepackage{caption}
\usepackage{subcaption}
\usepackage{dsfont}
\usepackage{mathtools}
\usepackage{bbm}
\newcommand{\der}{\operatorname{d\!}{}}
\usepackage{enumitem}
\usepackage{thmtools,thm-restate}

\newtheorem{corollary}{Corollary}
\newtheorem{example}{Example}
\newtheorem{lemma}{Lemma}
\newtheorem{property}{Property}
\newtheorem{proposition}{Proposition}
\newtheorem*{remark}{Remark}
\newtheoremstyle{break}
  {\topsep}{\topsep}%
  {\itshape}{}%
  {\bfseries}{}%
  {\newline}{}%
\theoremstyle{break}

\newtheorem{definition}{Definition}

\usepackage{wrapfig}

\newcommand{\Reals}{\mathbb{R}}

\newcommand{\Natural}{\mathbb{N}}

\newcommand{\Ball}{\mathfrak{B}}
\newcommand{\Prob}{\mathbb{P}}
\newcommand{\Expect}{\mathbb{E}}
\newcommand{\Loss}{\mathcal{L}}
\newcommand{\Empirical}{\mathcal{E}}
\newcommand{\Hypothesis}{\mathcal{H}}
\newcommand{\Lip}{\text{Lip}_1(\Reals^n,\Reals)}
\newcommand{\Params}{\theta}
\newcommand{\Znc}{\mathcal{Z}}
\newcommand{\Xsub}{\mathcal{X}}
\newcommand{\Labels}{\mathcal{Y}}
\newcommand{\Classifiers}{\mathcal{C}}
\newcommand{\BigO}{\mathcal{O}}
\newcommand{\sign}{\text{sign}}
\newcommand{\Jacobian}{J}
\newcommand{\indicator}{\mathds{1}}
\newcommand{\diam}[1]{\text{diam }#1}
\newcommand{\Wasserstein}{\mathcal{W}}
\newcommand{\supp}[1]{\text{supp }#1}

\newcommand*\closure[1]{\overline{#1}} 
\newcommand{\LipCl}{\text{LipNet1 }}
\newcommand{\LipInf}{\text{AllNet }}

\newcommand{\veryshortarrow}[1][3pt]{\mathrel{%
   \vcenter{\hbox{\rule[-.5\fontdimen8\textfont3]{#1}{\fontdimen8\textfont3}}}%
   \mkern-4mu\hbox{\usefont{U}{lasy}{m}{n}\symbol{41}}}}

\DeclareMathOperator*{\argmax}{arg\,max}
\DeclareMathOperator*{\argmin}{arg\,min}
\DeclarePairedDelimiter\ceil{\lceil}{\rceil}

\let\emptyset\varnothing
\newcommand{\defeq}{\vcentcolon=}

\newcommand{\yes}[1]{\textcolor{OliveGreen}{#1}}
\newcommand{\no}[1]{\textcolor{BrickRed}{#1}}

\usepackage[breaklinks=true,bookmarks=false]{hyperref}

\title{Pay attention to your loss: understanding misconceptions about 1-Lipschitz neural networks}

\author{Louis Béthune, \footnotemark[2]\\
IRIT, Université Paul-Sabatier\\
Toulouse, France
\And Thibaut Boissin, \footnotemark[2]\\
IRT Saint-Exupéry\\
Toulouse, France
\And Mathieu Serrurier\\
IRIT, Université Paul-Sabatier\\
Toulouse, France
\And Franck Mamalet\\
IRT Saint-Exupéry\\
Toulouse, France
\And Corentin Friedrich\\
IRT Saint-Exupéry\\
Toulouse, France
\And Alberto Gonz\'alez-Sanz\\
IMT, Université Paul-Sabatier\\
Toulouse, France
}

\begin{document}

\maketitle

\begin{abstract}

    Lipschitz constrained networks have gathered considerable attention in the deep learning community, with usages ranging from Wasserstein distance estimation to the training of certifiably robust classifiers. However they remain commonly considered as less accurate, and their properties in learning are still not fully understood. In this paper we clarify the matter: when it comes to classification 1-Lipschitz neural networks enjoy several advantages over their unconstrained counterpart. First, we show that these networks are as accurate as classical ones, and can fit arbitrarily difficult boundaries. Then, relying on a robustness metric that reflects operational needs we characterize the most robust classifier: the WGAN discriminator. Next, we show that 1-Lipschitz neural networks generalize well under milder assumptions. Finally, we show that hyper-parameters of the loss are crucial for controlling the accuracy-robustness trade-off. We conclude that they exhibit appealing properties to pave the way toward provably accurate, and provably robust neural networks.    

\end{abstract}

\section{Introduction}\label{sec:intro}
    1-Lipschitz neural networks have drawn great attention in the last decade, with motivation ranging from adversarial robustness to Wasserstein distance computation. In the following, we denote by \textbf{\LipCl} the class of 1-Lipschitz neural networks, by \textbf{\LipInf} the class of neural networks without constraints on their Lipschitz constant, i.e conventional neural networks.

Roughly speaking, the Lipschitz constant of neural networks quantifies how much their outputs can change when inputs are perturbed. When this constant is high, as it is often the case for neural networks of \LipInf, they become vulnerable to adversarial attacks (see~\cite{szegedy2013intriguing,yuan2019adversarial} and references therein): a carefully chosen small noise added to the inputs, usually imperceptible, can change the class prediction. One possible defense against adversarial attacks is to constrain the network to be 1-Lipschitz~(in \LipCl)~\cite{li2019preventing}, which provides provable robustness guarantees, together with an improvement of generalization~\cite{Sokolic_2017} and interpretability of the model~\cite{tsipras2019robustness}. \LipCl networks are also used to estimate Wasserstein distance, thanks to Kantorovich-Rubinstein duality in the seminal work of WGAN~\cite{arjovsky2017wasserstein}.  

  
Despite their competitiveness with networks of \LipInf on medium scale problems~\cite{cisse2017parseval,serrurier2020achieving}, they still suffer from misconceptions. A belief commonly invoked against networks of \LipCl is that they are less expressive: ``Lipschitz-based approaches suffer from some representational limitations that may prevent them from achieving higher levels of performance and being applicable to more complicated problems''~\cite{huster2018limitations}.

Although this claim seems rational at first glance, the link between Lipschitz constant and expressiveness is not trivial. While there is an obvious lack of expressiveness for regression tasks, this intuition fades when it comes to classification. Indeed, every \LipInf network $g:\Reals^n\rightarrow\Reals^K$ is $L$-Lipschitz for some (generally unknown) $L>0$. Then $f=\frac{1}{L}g$ is a \mbox{1-Lipschitz} neural network with the same decision boundary, since prediction $\argmax_k g_k$ is invariant by positive rescaling of the logits. In particular, $f$ has the same accuracy and also the same robustness to adversarial attacks as $g$. We illustrate this empirically by \textbf{training a \LipCl network until it reaches 99.96\% accuracy on CIFAR-100 with random labels} (see Appendix~\ref{section:cifar100_random}).  


We demonstrate that \LipCl networks are theoretically better grounded than \LipInf networks when it comes to classification, through our threefold contribution on Expressiveness (Section~\ref{sec:classificationpower}), Robustness (Section~\ref{sec:robustness}) and Generalization (Section~\ref{sec:generalize}).  

\textbf{First, in Section~\ref{sec:classificationpower}}  we confirm that \LipCl are as expressive as \LipInf networks for classification, and can learn arbitrary complex decision boundary. We show that hyper-parameters of the loss are of crucial importance, and control the ability to fit properly the train set.  
  
\textbf{Then, in Section~\ref{sec:robustness}} we show that accuracy and robustness are often antipodal objectives. We characterize the robustness of the highest accuracy \LipCl classifier: it is achieved by the Signed Distance Function~(Definition~\ref{def:nlctwoclasses} in Appendix \ref{app:classificationpower}). We also characterize the classifier of highest certifiable robustness, and we show it corresponds to the dual potential of Wasserstein-1 distance (i.e the discriminator of a WGAN~\cite{arjovsky2017wasserstein}).  
  
\textbf{Finally, in Section~\ref{sec:generalize}} we show that \LipCl benefit from several generalization guarantees. They are consistent estimators: contrary to \LipInf, we prove that their train loss will converge to test loss as the size of the train set increases. Moreover, we show that \LipCl classifiers with margin are PAC-learnable~\cite{valiant1984theory}: it provides bounds on the number of train examples required to reach a targeted test accuracy. Interestingly, this bound is independent of the architecture size, which allows to train enormous \LipCl networks without risking overfitting.

\vspace{-2mm}
\section{Notations and experimental setting}\label{sec:notations}
\vspace{-2mm}
    The core of the paper mainly deal with  binary classification over $\Reals^n$ with label set $\Labels=\{-1,+1\}$. Let $(X,Y)$ be a random variable taking values on $\Xsub\times\Labels$, where $\Xsub\subset\Reals^n$ is assumed to be a compact set. Such a pair follows the joint distribution $\Prob_{XY}$, defined on the space of probability measures $\mathcal{P}(\Xsub\times\Labels)$. The marginal distribution of $X$ is denoted by $\Prob_X\in\mathcal{P}(\Xsub)$ and its support by $\supp \Prob_X$. We suppose the observation of a sample $(x_1,y_1), \dots, (x_p,y_p)$ i.i.d. with common law $\Prob_{XY}$, and the goal is to learn a classifier $c:\Xsub\rightarrow\Labels$ modeling the optimal Bayes classifier $\argmax_{y\in\Labels}\Prob_{Y|X}(y|x)$. $P$ (resp. $Q$) denotes the input distribution of label $+1$ (resp. $-1$).  
  
The Lipschitz constant $\text{Lip}(f)$ of a function $f:\Reals^n\rightarrow\Reals^K$ is defined as the smallest $L\geq 0$ such that for all $x,z\in\Reals^n$ we have $\|f(x)-f(z)\|\leq L\|x-z\|$. In the rest of the paper, we focus on euclidean norm $\|\cdot\|$ for vectors and spectral norm $\|\cdot\|_2$ for matrices. The set of $L$-Lipschitz functions over $\Xsub\subset\Reals
^n$ with image in $\Reals^K$ is denoted $\text{Lip}_L(\Xsub,\Reals^K)$. 
  
\begin{definition}[Class of \LipInf networks]
\LipInf denotes the set of unconstrained neural networks. It includes any feed-forward network of fixed depth (without recurrent mechanisms) using affine layers (including convolutions and batch normalization) with weight matrices $W_1,W_2,\ldots W_d$ 
and Lipschitz activation function $\sigma$ (such as ReLU, sigmoid, tanh, etc). No constraint is enforced on their Lipschitz constant during training.  
\end{definition}

\begin{definition}[Class of \LipCl networks]
\LipCl denotes the set of feed-forward neural networks $f$ defined as in Theorem 3 of Anil et al.~\cite{anil2019sorting}: $\|W_1\|_{2\rightarrow\infty}\leq 1$ (see~\cite{cape2019two} for details on the mixed norm $\|\cdot\|_{2\rightarrow\infty}$) and $\|W_i\|_{\infty}\leq 1$ for $i\geq 2$, and GroupSort2 activation function. They fulfill $\text{Lip}(f)\leq 1$.  
\end{definition}  
 \begin{remark}
 \LipInf networks benefit from universal approximation theorem in $C(\Xsub,\Reals^K)$, a classical result of literature~\cite{hassoun1995fundamentals}. \LipCl networks also benefit from an universal approximation theorem in $\text{Lip}_1(\Xsub,\Reals)$ with respect to uniform convergence~\cite{anil2019sorting}. Note that $\text{Lip}_L(\Xsub,\Reals^K)=\{Lf\mid f\in\text{Lip}_1(\Xsub,\Reals^K)\}$ so \LipCl can be used to approximate functions in $\text{Lip}_L(\Xsub,\Reals^K)$.
 \end{remark}
  
In practice authors of~\cite{anil2019sorting} noticed that using orthogonal weight matrices (i.e $W_i^TW=I$) yielded the best results. All our experiments use the Deel.Lip\footnote{\url{https://github.com/deel-ai/deel-lip} distributed under MIT License (MIT).} library~\cite{serrurier2020achieving}, following ideas of~\cite{anil2019sorting}. The networks use 1) orthogonal weight matrices and 2) GroupSort2 activations~\cite{anil2019sorting}. Orthogonalization is enforced using Spectral normalization~\cite{miyato2018spectral} and Bj{\"o}rck algorithm~\cite{bjorck1971iterative}. 
These networks belong to \LipCl by construction (see Appendix~\ref{ap:deellip} for our choice of architecture and relevant related work).   
  
  \LipCl networks provide robustness radius certificates against adversarial attacks~\cite{NEURIPS2018_48584348}. Computing these certificates is straightforward and does not increase runtime, contrary to methods based on bounding boxes or abstract interpretation~\cite{latorre2019lipschitz,weng2018towards,weng2018evaluating,zhang2021towards}. There is no need for adversarial training~\cite{madry2018towards} that fails to produce guarantees, or for randomized smoothing~\cite{pmlr-v97-cohen19c} which is costly.  
    
  Confusingly, any network of \LipInf has a finite Lipschitz constant, but computing it is NP-hard~\cite{scaman2018lipschitz}. Only a loose upper bound can be cheaply estimated: $\text{Lip}(f)\leq \text{Lip}(\sigma)^d\Pi_{i=1}^d \|W_i\|_2$ using the property that $\text{Lip}(f_d\circ f_{d-1}\circ\ldots \circ f_1)\leq\Pi_{i=1}^d\text{Lip}(f_i)$. In practice, this bound is often too high to provide meaningful certificates and besides, \LipInf networks have usually very small robustness radius~\cite{szegedy2013intriguing}.  
      
    \begin{definition}[Adversarial Attack] 
        For any classifier $c:\Xsub\rightarrow\Labels$, any $x\in\Reals^n$, consider the following optimization problem:
        \begin{equation}
            \epsilon=\min_{\delta\in\Reals^n} \|\delta\|\text{ such that }c(x+\delta)\neq c(x).
        \end{equation}  
        $\delta$ is an \textit{adversarial attack}, $x+\delta$ is an \textit{adversarial example}, and $\epsilon$ is the \textit{robustness radius} of $x$. 
    \end{definition}  
      
    \begin{property}[Local Robustness Certificates~\cite{NEURIPS2018_48584348}]\label{thm:certificates}
        For any $f\in\LipCl$ the robustness radius $\epsilon$ of binary classifier $\sign\circ f$ at example $x$ verifies $\epsilon\geq|f(x)|$.
    \end{property} 
  
\paragraph{Losses:} The Binary Cross-Entropy (BCE) loss (also called logloss) is among the most popular choices of loss within the deep learning community. Let $f:\Reals^n\rightarrow\Reals$ a neural network. For an example $x\in\Reals^n$ with label $y\in\Labels$, and $\sigma(x)=\frac{1}{1+\exp{(-x)}}$ the logistic function mapping logits to probabilities, the BCE is written $\Loss^{bce}_{\tau}(f(x),y)=-\log{\sigma(y\tau f(x))}$, with temperature scaling parameter $\tau>0$. This hyper-parameter of the loss defaults to $\tau=1$ in most frameworks such as Tensorflow or Pytorch. Note that $\Loss^{bce}_{\tau}(f(x),y)=\Loss^{bce}_{1}(\tau f(x),y)$ so we can equivalently tune $\tau$ or the Lipschitz constant $L$. We show in Section~\ref{sec:consistency} that \textbf{for \LipCl the temperature $\tau$ allow to control the generalization gap}. We also consider the Hinge loss $\Loss^{H}_{m}(f(x),y)=\max{(0, m-yf(x))}$ with margin $m>0$, as used in~\cite{li2019preventing} for \LipCl networks training.

We focus on binary classification for readability and clarity purposes; however, we prove in Appendices~\ref{app:mutliclassexpressive} and~\ref{app:multiclasshkr} that \textbf{the following theoretical results generalize to the multi-class case}, as done in experiments. The proofs of all propositions can be found in the appendix. 
\section{1-Lipschitz classifiers are expressive}\label{sec:classificationpower}
    In this section, we show that \LipCl are as powerful as any other classifier, like their unconstrained counterpart. In particular, when classes are separable they can achieve 100\% accuracy. 

\subsection{Boundary decision fitting}

    \begin{restatable}{proposition}{lippowerbinary}{\normalfont\textbf{Lipschitz Binary classification.}}\label{thm:lippowerbinary}
        For any binary classifier $c:\Xsub\rightarrow\Labels$ with closed pre-images ($c^{-1}(\{y\})$ is a closed set) there exists a 1-Lipschitz function $f:\Reals^n\rightarrow\Reals$ such that $\sign{(f(x))}=c(x)$ on $\Xsub$ and such that $\|\nabla_x f\|=1$ almost everywhere (w.r.t Lebesgue measure). 
    \end{restatable} 
    
      
    The level-sets of a $\text{Lip}_1(\Xsub,\Reals^K)$ functions (and especially the decision boundary) can be arbitrarily complex: restraining classifiers to $\text{Lip}_1(\Xsub,\Reals)$ does not affect the classification power. 
    \begin{definition}[\textbf{$\epsilon$-separated distributions}]
    Distributions $P$ and $Q$ are $\epsilon$-separated if the distance between $\supp P$ and $\supp Q$ exceeds $\epsilon>0$.
    \end{definition}
    \begin{restatable}{corollary}{zeroerror}{\normalfont\textbf{Separable classes implies zero error}.}\label{thm:zeroerror}
        If $P$ and $Q$ are $\epsilon$-separated, then there exists a network $f\in\LipCl$ such that \textbf{error} $E(\sign\circ f)\coloneqq\Expect_{(x,y)\sim\Prob_{XY}}[\indicator\{\sign(f(x))\neq y\}]=0$.  
    \end{restatable} 
    
      The class of \LipCl networks does not suffer from bias for classification tasks. Some empirical studies show that indeed most datasets classes are separable~\cite{yang2020closer} such as CIFAR10 or MNIST. Furthermore, even if the classes are not separable, functions of \LipCl can nonetheless approximate the optimal Bayes classifier. Lipschitz constraint is not a constraint on the shape of the boundary (Figure~\ref{fig:signeddistancefunction}), but on the slope of the landscape of $f$. 
     
    
    \begin{figure}
     \centering
     \begin{subfigure}[b]{0.48\linewidth}
         \centering
         \includegraphics[width=0.9\linewidth]{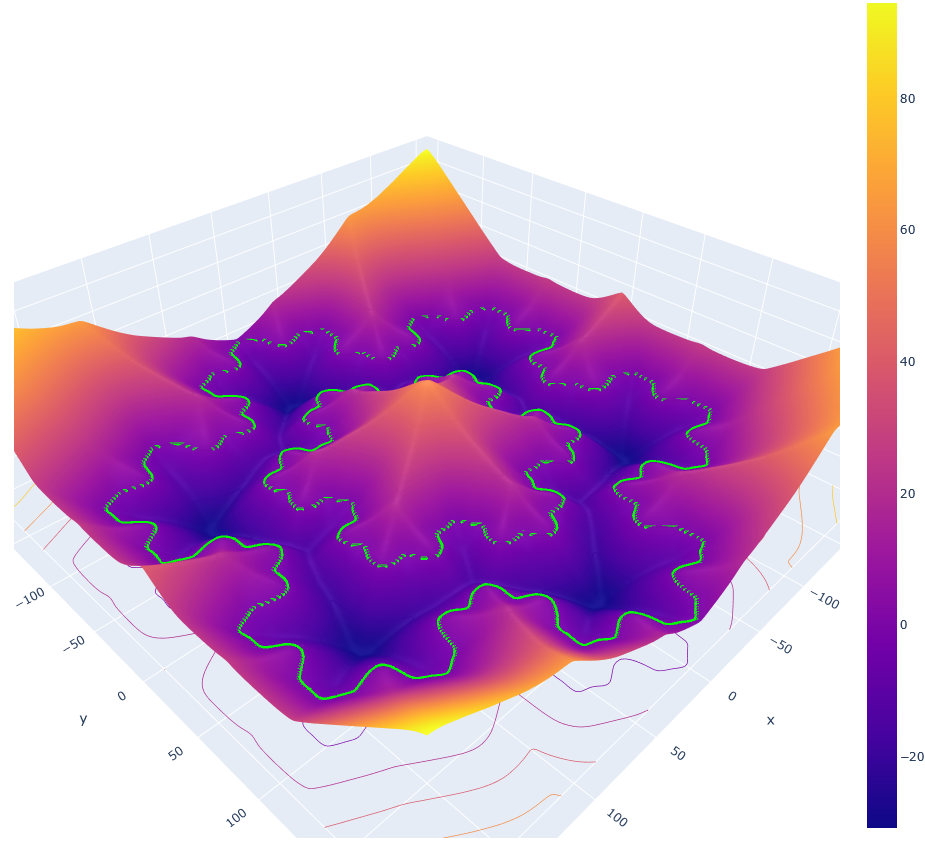}
         \caption{\textbf{Complex Decision Boundary $\partial$}. We chose $\partial$ as the fourth iteration of Von Koch Snowflake. We chose $P$ as the interior ring, while the center and the exterior correspond to $Q$. We train a \LipCl network with MSE to fit the SDF (Definition~\ref{def:nlctwoclasses} in Appendix) ground truth (160 000 pixels), until MAE is inferior to 1. It proves empirically that \LipCl networks can handle very sharp (almost fractal) decision boundary.}
         \label{fig:signeddistancefunction}
     \end{subfigure}
     \hfill
     \begin{subfigure}[b]{0.48\linewidth}
         \centering
         \includegraphics[width=1.\linewidth]{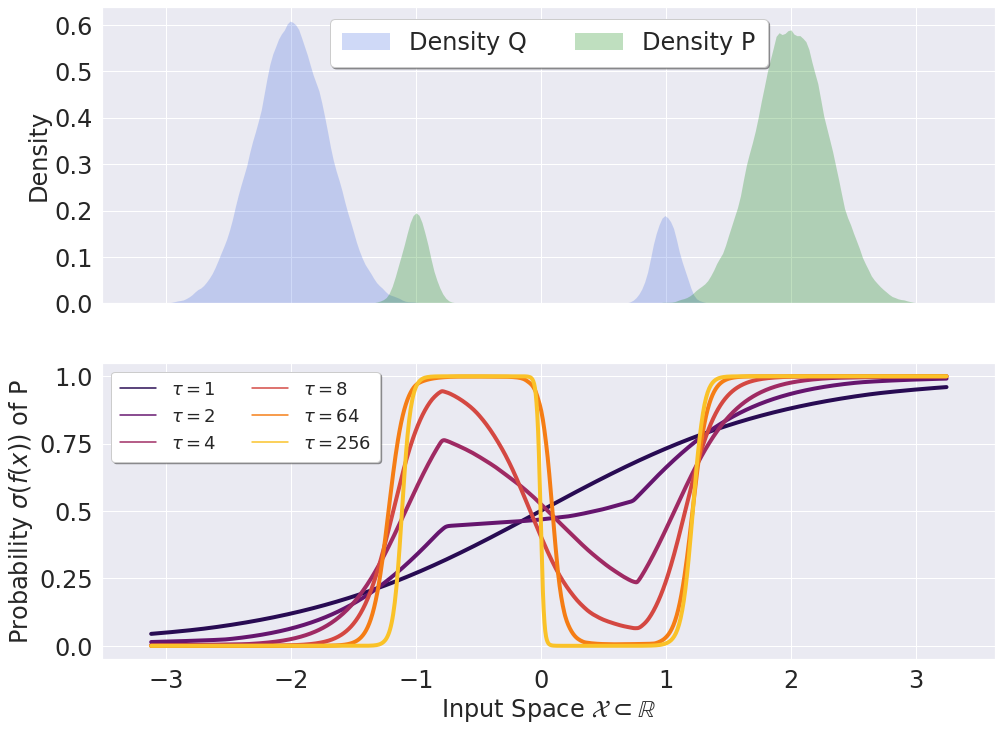}
         \caption{\textbf{Importance of $\tau$ in BCE}. We train a \LipCl network with BCE and different values for $\tau$. We chose a toy example where $P$ and $Q$ are Gaussian mixtures with two modes of weights $0.9$ and $0.1$. We highlight the different shapes of the minimizer $\sigma\circ f$ as function of $\tau$. \textbf{High values of $\tau$ leads to better fitting, whereas for lower $\tau$ the small weights Gaussian of the mixture are treated as noise and ignored}.}
         \label{ex:nosameminimum}
     \end{subfigure}
     \caption{}
    \end{figure}

\subsection{Understanding why \LipCl are often perceived as not expressive}\label{sec:lipclminexists}

    \LipCl networks cannot reach zero loss with BCE: this may explain why they are perceived as not expressive enough. Yet the minimizer of BCE exists and is well defined.
    
    \begin{restatable}{proposition}{minimizerattained}{\normalfont\textbf{BCE minimization for \mbox{1-Lipschitz} functions}.}\label{thm:minimizerattained}
        Let $\Xsub\subset\Reals^n$ be a compact and $\tau>0$. Then the infimum in Equation~\ref{eq:optim1lip} is a minimum, denoted $f^{\tau}\in\text{Lip}_1(\Xsub,\Reals)$:
        \begin{equation}\label{eq:optim1lip}
            f^{\tau}\in\arg\inf_{f\in\text{Lip}_1(\Xsub,\Reals)}\Expect_{(x,y)\sim\Prob_{XY}}[\Loss^{bce}_{\tau}(f(x),y)].
        \end{equation}
         Moreover, the \LipCl networks will not suffer of vanishing gradient of the loss (see Appendix~\ref{app:gradientpreserving}).
    \end{restatable}
    
    Machine learning practitioners are mostly interested in maximizing accuracy.
    However, the minimizer of BCE is not necessarily a minimizer of the error (see Figure~\ref{ex:nosameminimum}). Yet, BCE is notoriously a differentiable proxy of the error $E(\sign\circ f)$, and as $\tau\rightarrow\infty$ we get asymptotically closer to maximum empirical accuracy. Bigger value for $\tau$ might ultimately lead to overfitting, playing the same role as the Lipschitz constant $L$ (see Figure~\ref{ex:nosameminimum}).  
    
      
     \textbf{The implicit parameter $\tau=1$ of the loss is partially responsible of the poor accuracy of \LipCl networks in literature}, and not by any means the hypothesis space \LipCl itself. This can be observed in practice : when temperature $\tau$  (resp. margin $m$) of cross-entropy (resp. hinge loss) is correctly adjusted a small \LipCl CNN can reach a competitive \textbf{88.2\% validation accuracy on the CIFAR-10 dataset} (results synthetized and discussed in Figure~\ref{fig:pareto_hkr}) \textit{without} residual connections, batch normalization or dropout. Conversely, \LipInf networks are roughly equivalent to learning a \LipCl network with $\tau\rightarrow\infty$: without regularization or data augmentation, such a network can always reach 100\% train accuracy without generalization guarantees. 
      
\section{1-Lipschitz classifiers are certifiably robust}\label{sec:robustness}
    \paragraph{Is there a trade-off between accuracy and robustness?}

    Although the existence of a trade-off between accuracy and robustness is commonly admitted, some works argue that ``Robustness is not inherently at odds with accuracy''\cite{yang2020closer}. We propose a unified consideration by stating that for a given train accuracy, robustness can be maximized up to a certain point, but allowing a lower train accuracy helps achieving a higher robustness. Finally one must keep in mind that this trade-off lives in the shade of generalization (see Section~\ref{sec:generalize}).

    \begin{figure}
        \centering
        \includegraphics[width=0.8\textwidth]{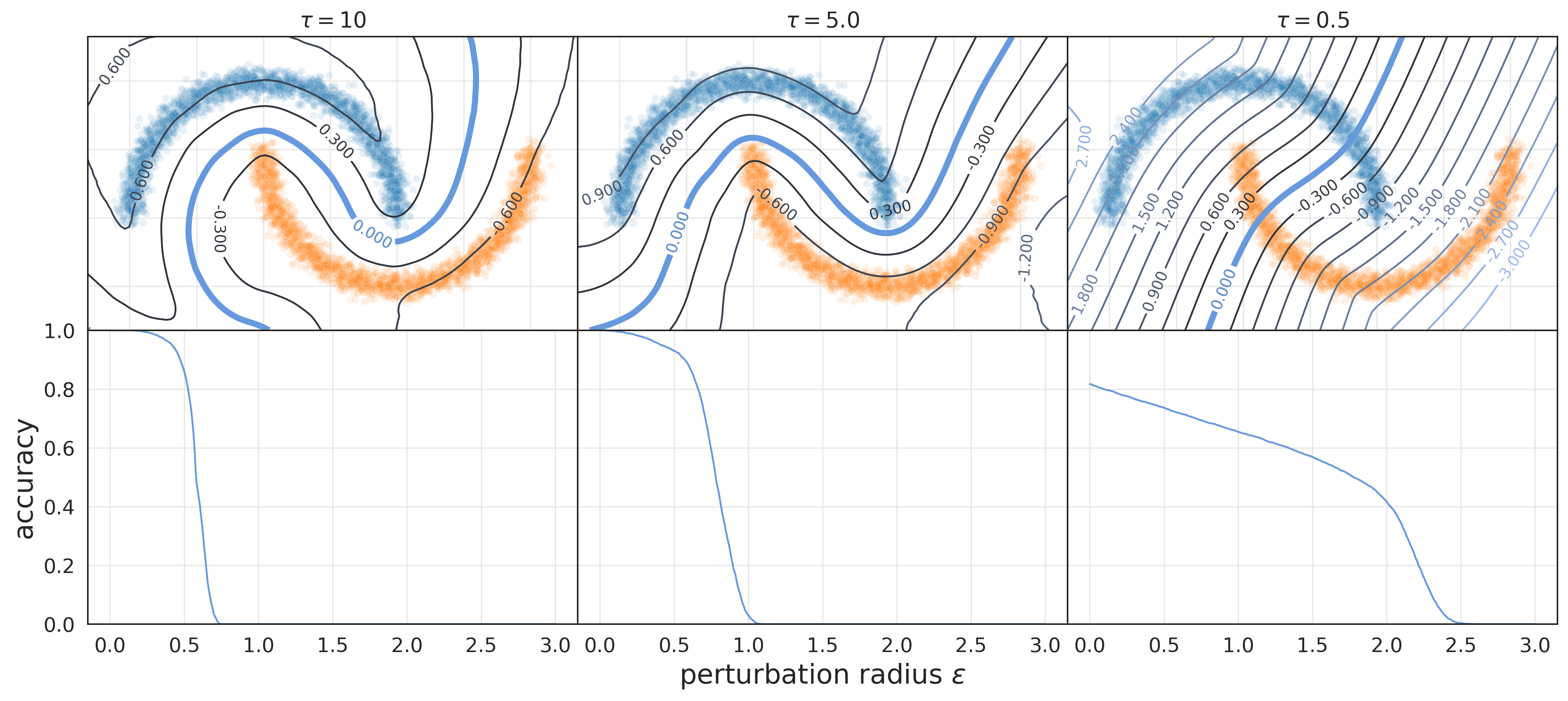}
        \caption{\textbf{Accuracy-robustness tradeoff:} Each network is optimal with respect to a certain criterion. The leftmost network is the most accurate at robustness radius $\epsilon\leq 0.3$, the rightmost maximizes the MCR at the cost of low clean accuracy. The center network corresponds to a compromise.
        }
        \label{fig:twomoonstradeoff}
    \end{figure}

\subsection{Improving the robustness of the maximally accurate classifier}
    \label{max_acc}
      
    The Signed Distance Function~\cite{rousson2002shape} (SDF) (see Definition~\ref{def:nlctwoclasses} in Appendix~\ref{app:classificationpower}) associated to the frontier $\partial$ of Bayes classifier $b$ is the 1-Lipschitz function that provides the largest certificates among the classifiers of maximum accuracy. Moreover, those certificates are exactly equal to the distance of adversarial samples. Iterative gradient based attacks (see~\cite{chakraborty2018adversarial} and references therein) can succeed in one step: far from being a weakness, this may improve the interpretability of the model~\cite{dong2017towards,tomsett2018failure,ross2018improving}. 
    
      
    \begin{restatable}{corollary}{thightboundcertificate}\label{thm:thightboundcertificate}
        For the SDF($b$), the bound of Property~\ref{thm:certificates} is tight: $\epsilon=|f(x)|$. In particular $\delta=-f(x)\nabla_xf(x)$ is guaranteed to be an adversarial attack. The risk is the smallest possible. There is no classifier with the same risk and better certificates. Said otherwise the SDF($b$) is the solution to:
        \begin{equation}
            \begin{aligned}
                \max_{f\in\Lip}\min_{x\in\Xsub}&\min_{\substack{\delta\in\Reals^n\\ \sign{(f(x+\delta))}\neq \sign{(f(x))}}} \|\delta\|,\\
                \textrm{under the constraint } &  f\in\argmin_{g\in\Lip}E(\sign\circ g).\\
            \end{aligned}
        \end{equation}
    \end{restatable} 
    
    The SDF($b$) cannot be explicitly constructed since it relies on the (unknown) optimal Bayes classifier. 
    
\subsection{Improving the accuracy of the maximally robust classifier}
    \label{max_rob}
    
    On the opposite side, we exhibit a family of classifiers with lower accuracy but with higher certifiable robustness. We insist that the quantity of interest is the \textit{certifiable robustness} $|f(x)|$ and not the \textit{true empirical robustness} $\epsilon$ (which can be higher). The former is computed exactly and freely, while the latter is a difficult problem for which only upper bounds returned by attacks are available. In the literature, the robustness is only evaluated on well classified examples. The certificate can be both interpreted as a form of ``confidence'' of the network, and as the minimal perturbations required to switch the class. Hence, we shall weight negatively this certificate for the examples that are misclassified since confidence in presence of errors is worse. For this reason, we propose in Definition~\ref{def:meancertifiable} a new metric called the Mean Certifiable Robustness (MCR).
      
    \begin{definition}[\textbf{Mean Certifiable Robustness -- MCR}]\label{def:meancertifiable}
        For any function $f:\Xsub\rightarrow\Reals\in\LipCl$ we define its weighted mean certifiable robustness $\mathcal{R}_{(P,y)}(f)$ on class $P$ with label $y$ as:  
        \begin{equation}
            \begin{aligned}
                \mathcal{R}_{(P,y)}(f)&\defeq\Expect_{x\sim P}[{\mathds{1}\{yf(x)>0\}|f(x)|}]+\Expect_{x\sim P}[{-\mathds{1}\{yf(x)<0\}|f(x)|}]=\Expect_{x\sim P}{yf(x)}.
            \end{aligned}
        \end{equation} 
    \end{definition}  
      
    We can readily see from the definition that the classifier with highest MCR for class $P$ is the constant classifier $f=y\times\infty$. The interest of this notion arises when we consider minimizing the loss function $\Loss^{W}(f(x),y)\defeq-yf(x)$, i.e when looking for classifier with the highest MCR.   
      
    \begin{restatable}{property}{mcriswass}{\normalfont\textbf{Wasserstein classifiers (i.e WGAN discriminators) are optimally robust}.}\label{thm:mcriswass}
        The minimum of $\Loss^{W}(f(x),y)$ over $P$ and $Q$ is the Wasserstein-1 distance~\cite{villani2008optimal} between $P$ and $Q$ according to the Kantorovich-Rubinstein duality:
        \begin{equation}\label{eq:wassdef}
            \max_{f\in\text{Lip}_1(\Xsub,\Reals)}\mathcal{R}_{(P,+1)}(f)+\mathcal{R}_{(Q,-1)}(f)=\min_{f\in\Lip}\Expect_{\Prob_{XY}}[\Loss^{W}(f(x),y)]=\Wasserstein_1(P,Q).
        \end{equation}
    \end{restatable}  
    
    Even though the minimizer of $\Loss_{W}(f(x),y)$ can have low accuracy, it has the highest MCR. Interestingly, the minimizer $f^{*}$ of equation~\ref{eq:wassdef} is invariant by translation: $f^{*}-T$ is also a minimizer for any $T\in\Reals$. When $T\rightarrow\infty$ (resp. $-\infty$) the classifier has $100\%$ recall on $Q$ (resp. $P$), and $0\%$ on $P$ (resp. $Q$). Does it always exist $T^{*}$ with $100\%$ accuracy overall? Sadly, even when the $P$ and $Q$ have disjoint support, the answer is no. We precise this empirical observation of~\cite{serrurier2020achieving} in Proposition~\ref{thm:weakwass}.
      
    \begin{restatable}{proposition}{weakwass}{\normalfont\textbf{WGAN discriminators are weak classifiers}.}\label{thm:weakwass}
        For every $\frac{1}{2}\geq\epsilon>0$ there exist distributions $P$ and $Q$ with disjoint supports in $\Reals$ such that for any optimum $f$ of equation~\ref{eq:wassdef}, the error of classifier $\sign\circ f$ is superior to $\frac{1}{2}-\epsilon$.
    \end{restatable}  
    Note that this minimum also invariant by dilatation: any \textit{finite} upper bound $L$ can be chosen for Equation~\ref{eq:wassdef} (see Appendix~\ref{sec:krunit}).
      

    \subsection{Controlling the accuracy/robustness tradeoff with loss parameters}
    
    Now that the extrema of the accuracy robustness tradeoff were characterized in \ref{max_acc} and \ref{max_rob}, is yet to be answered if it is possible to control this tradeoff using conventional loss (and its parameters, as introduced in \ref{sec:lipclminexists}).
    
    Interestingly, observe that $\Loss^{bce}_{\tau}(f(x),y)=\log{2}-\frac{y\tau f(x)}{2}+\mathcal{O}(\tau^2 f^2(x))$ so when $\tau\rightarrow 0$ we get:
        $$\min_{f\in\text{Lip}_1(\Xsub,\Reals)}\frac{4}{\tau}\left(\Expect_{(x,y)\sim P_{XY}}[\Loss^{bce}_{\tau}(f(x),y)]-\log{2}\right)=-\Wasserstein_1(P,Q).$$
    In the limit of small temperatures, the BCE minimizer essentially behaves like the classifier of the highest MCR (see Figure~\ref{fig:pareto_hkr} and Appendix~\ref{ap:bceot}). Similarly, the HKR loss $\Loss^{hkr}$ introduced in~\cite{serrurier2020achieving} for \LipCl training allows fine grained control of the accuracy-robustness tradeoff: 
    \begin{equation}~\label{eq:hkrdef}
        \Loss^{hkr}_{m,\alpha}(f(x),y)=\Loss^W(f(x),y)+\alpha\Loss^H_{m}(f(x),y)=-yf(x)+\alpha\max{(0, m - yf(x))}.
    \end{equation}
    
    We recover $\Wasserstein_1$ behavior for $\alpha=0$, and hinge $\Loss^H_{m}$ behavior for $\alpha\rightarrow\infty$, in a fashion that reminds the role of $\tau$ for $\Loss^{bce}$.
    
    A key takeaway is that BCE, HKR and hinge loss have parameters that allow to control the accuracy robustness tradeoff, reaching on one side the maximum robustness of MCR, and the accuracy of unconstrained networks on the other. Empirically this tradeoff is observed as a Pareto front with accuracy on one axis, and robustness on the other. Figure \ref{fig:pareto_hkr} shows this on the CIFAR10 dataset using the $\epsilon$ robustness as robustness measures (other robustness measure yield similar observations, see fig \ref{fig:pareto-avg} and \ref{fig:pareto-MCR}).
    
    In conclusion, these last two sections demonstrate that restraining networks to be in \LipCl does not impact the classification capabilities while providing certificates of robustness; however, for these networks the loss parameters play an important role in this trade-off. 

    \begin{figure}
        \centering
        \includegraphics[width=0.95\textwidth]{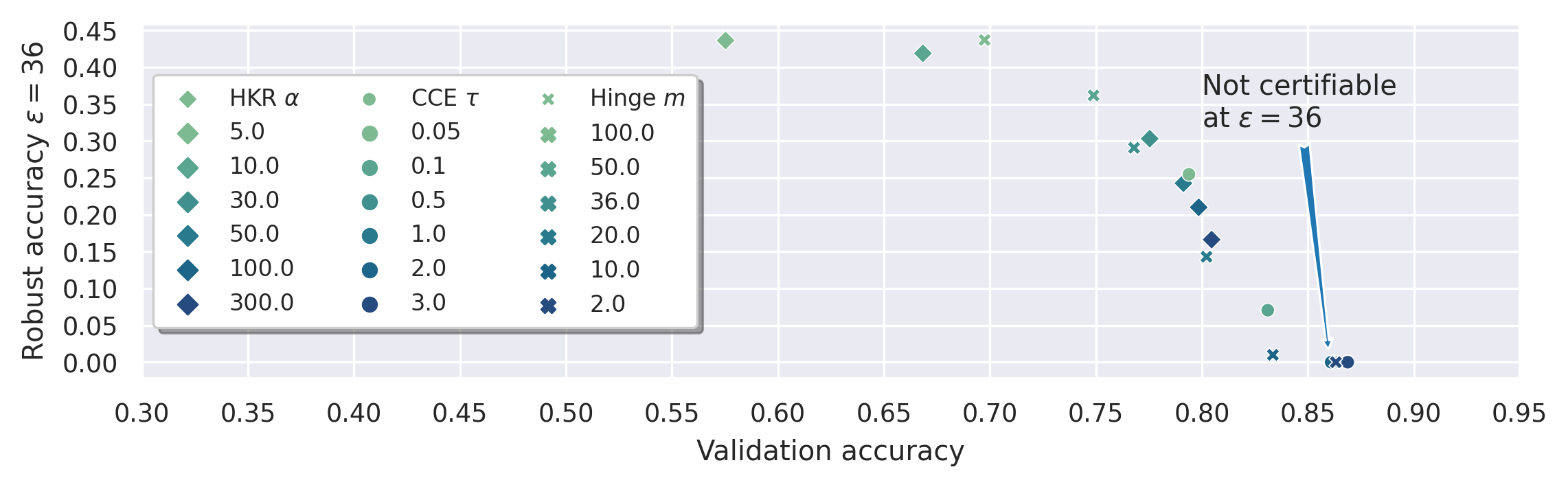}
        \caption{\textbf{Accuracy-Robustness trade-off on CIFAR10 with Hinge, HKR and Categorical Cross-Entropy (CCE) hyper-parameters.} Overall, for a given network architecture, a Pareto front appears between clean accuracy and robust accuracy. We move along it by tuning the parameters of each loss. We trained small \LipCl \textit{CNNs} (0.4M params) with basic data augmentation (see appendix~\ref{ap:pareto} for detailed experimental setting).}
        \label{fig:pareto_hkr}
    \end{figure}

\vspace{-2mm}
\section{1-Lipschitz classifiers have generalization guarantees}\label{sec:generalize}
    In this section, we explore the statistical and optimization properties of \LipCl networks, and we prove the assumption of~\cite{gouk2021regularisation} that ``adjusting the Lipschitz constant of a feed-forward neural network controls how well the model will generalise to new data''. 
    \vspace{-1mm}
    \subsection{Consistency of \LipCl class}\label{sec:consistency}
  
        \LipCl class enjoys another remarkable property since it is a Glivenko-Cantelli class: minimizers of Lipschitz losses are consistent estimators. In other words, as the size of the training set increases, the training loss becomes a proxy for the test loss: \LipCl neural networks will not overfit in the limit of (very) large sample sizes.  
          
        \begin{restatable}{proposition}{consistenceestimator}{\normalfont\textbf{Train Loss is a proxy of Test Loss}.}\label{thm:consistenceestimator}
            Let $\Prob_{XY}$ a probability measure on $\Xsub\times\Labels$ where $\Xsub\subset\Reals^n$ is a bounded set. Let $(x_i,y_i)_{1\leq i\leq p}$ be a sample of $p$ iid random variables with law $\Prob_{XY}$. Let $\Loss$ be a Lipschitz loss function over $\Reals\times\Labels$. We define:
            \begin{equation}
                \Empirical_p(f)\defeq\frac{1}{p}\sum_{i=1}^p\Loss(f(x_i),y_i)\text{ and }\Empirical_{\infty}(f)\defeq\Expect_{(x,y)\sim \Prob_{XY}}[\Loss(f(x),y)].
            \end{equation}
            Then the empirical loss $\Empirical_p(f)$ converges to the test loss $\Empirical_{\infty}(f)$ (taking the limit $p\rightarrow\infty$):
            \begin{equation}
                \min_{f\in\text{Lip}_L(\Xsub,\Reals)}\Empirical_p(f)\xrightarrow{a.s}\min_{f\in\text{Lip}_L(\Xsub,\Reals)}\Empirical_{\infty}(f).
            \end{equation}
        \end{restatable}
          
        It is another flavor of the bias-variance trade-off in learning. Thanks to Corollary~\ref{thm:zeroerror} we know the \LipCl class does not suffer of bias, while the generalization gap (i.e the variance) can be made as small as we want by increasing the size of the training set (see Figure~\ref{fig:consistency}). The number of examples required to close the generalization gap is dataset specific in general, however it seems that with low $\tau$ fewer examples are required. This result may seem obvious, but we emphasize \textbf{this property is not shared by \LipInf networks} (see Proposition~\ref{thm:lipinfoverfitter} in Appendix~\ref{app:lipinfnotrobust}). Nonetheless, most practitioners take for granted that bigger training sets ensure generalization for \LipInf networks.       
          
        \begin{figure}
        \centering
        \includegraphics[width=0.9\textwidth]{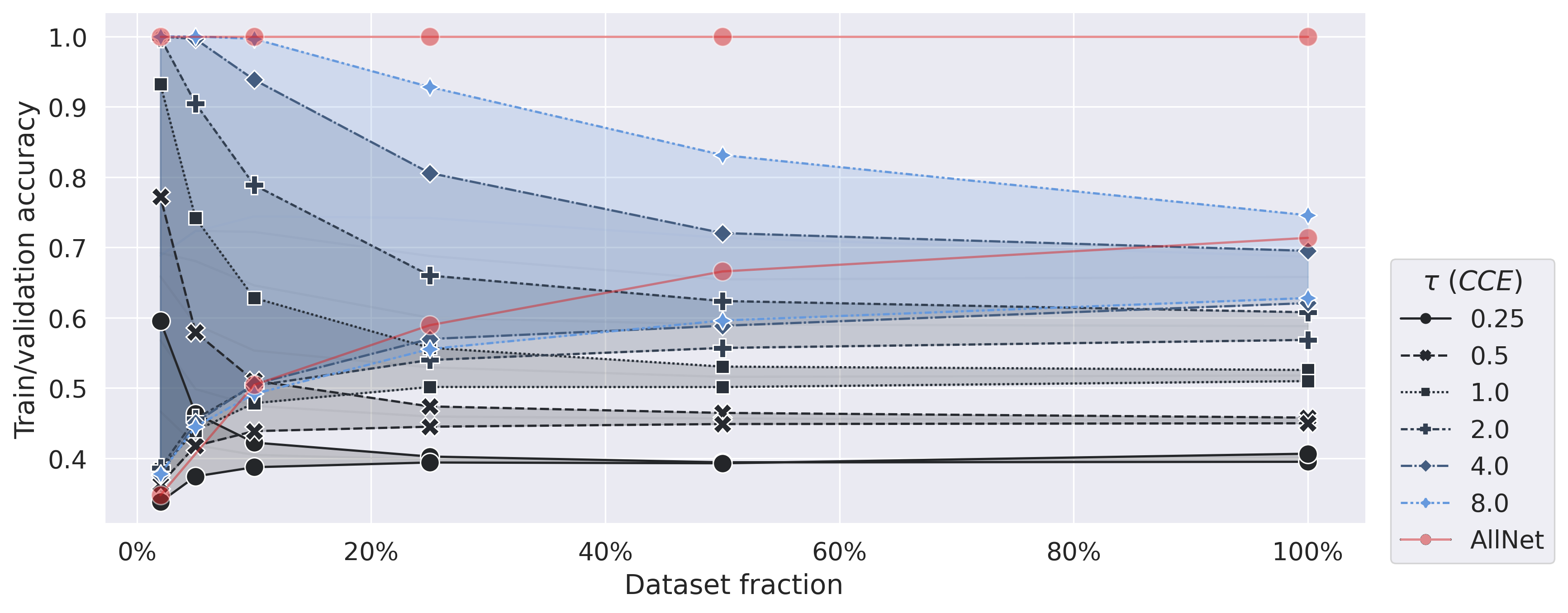}
        \caption{\textbf{Link between \LipCl and generalization gap, dataset size and cross-entropy temperature.} We train a CNN on different fractions of the CIFAR10 train set (2\%, 5\%, 10\%, 25\%, 50\% and 100\% on $x$-axis) with different values of temperature $\tau$ (highlighted by different colors). Train (resp. validation) accuracy forms the upper (resp. lower) bound of each envelope. As $\tau$ increases, more samples are required to reduce the generalization gap. Conversely, training a \LipCl network with small $\tau$ is equivalent to training a Lipschitz network with small $L$: the network generalizes well but the accuracy reaches a plateau (under-fitting). The \LipInf network (in red) severely overfit: the generalization gap is large and validation accuracy corresponds to the limit that would reach a \LipCl as $\tau$ increases. See appendix \ref{ap:consistency} for detailed experimental setting.}
        \label{fig:consistency}
        \end{figure}
    
    \subsection{Understanding why unconstrained networks are prone to overfitting}\label{sec:lipinfnotrobust}

        Surprisingly, on $\LipInf$networks, minimization of BCE leads to uncontrolled growth of Lipschitz constant and saturation of the predicted probabilities. This is an impediment to generalization results.     
          
        \begin{restatable}{proposition}{saturatedlimit}{\normalfont\textbf{Optimizing BCE over \LipInf leads to divergence}.}\label{thm:saturatedlimit}
            Let $f_t$ be a sequence of neural networks, that minimizes the BCE over a non-trivial training set (at least two different examples with different labels) of size $p$, i.e assume that:
            \begin{equation}
                \lim_{t\rightarrow\infty}\frac{1}{p}\sum_{i=1}^p\Loss_{\tau}(f_t(x_i),y_i)=0.
            \end{equation}
            Let $L_t$ be the Lipschitz constant of $f_t$. Then $\lim_{t\rightarrow\infty}L_t=+\infty$. There is at least one weight matrix $W$ such that $\lim_{t\rightarrow\infty}\|W_t\|=+\infty$. Furthermore, the predicted probabilities are saturated:
            \begin{equation}
                \lim_{t\rightarrow\infty}\sigma{(f_t(x_i))}\in\{0,1\}.
            \end{equation}
        \end{restatable}   
          
        This issue is especially important since Lipschitz constant and adversarial vulnerabilities are related~\cite{nar2019cross}. The predicted probability $\sigma(f(x))$ will either be $0$ or $1$ (regardless of the train set), which do not carry any useful information on the true confidence of the classifier
        \vspace{-1mm}
        \begin{figure}[ht]
        \begin{minipage}{0.55\linewidth}
        \begin{example}\label{ex:gotcha}
            Consider a classification task on $\Reals$ with linearly separable inputs $\{-1, 1\}$ and labels $\{-1, 1\}$. We use an affine model $f(x)=Wx+b$ for the logits (with $W\in\Reals$ and $b\in\Reals$) (one-layer neural network). It exists $\bar{W},\bar{b}$ such that $f$ achieves $100\%$ accuracy. However, as noticed in \cite{bishop2006pattern}~(Section~4.3.2) the BCE loss \textit{will not} be zero. The minimization occurs only with the diverging sequence of parameters $(\lambda\bar{W},\lambda\bar{b})$ as $\lambda\rightarrow\infty$. It turns out the infimum is not a minimum!
        \end{example} 
        \end{minipage}
        \begin{minipage}{0.45\linewidth}
        \centering
            \includegraphics[width=0.9\linewidth]{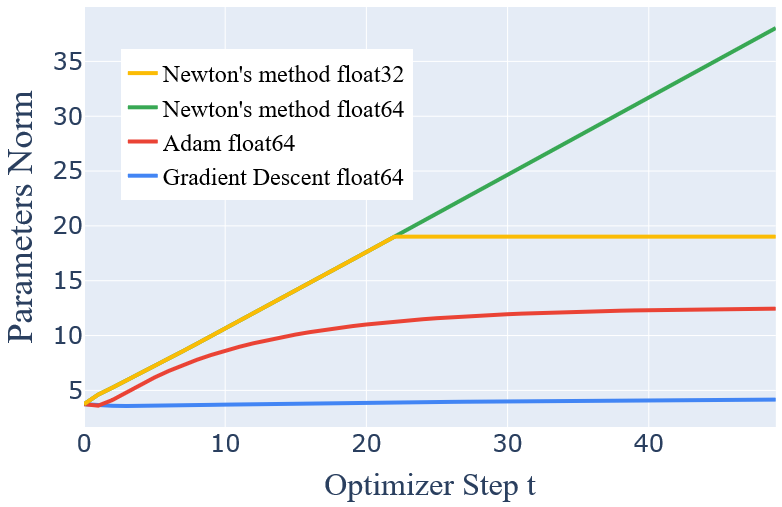}
            \caption{}
        \end{minipage}
        \end{figure}
           
        Even on toy example~\ref{ex:gotcha} with a trivial model, the minimization problem is ill-defined. Without weight regularization, the minimizer can not be attained. This is compliant with the high Lipschitz constant of \LipInf networks that have been observed in practice~\cite{scaman2018lipschitz}, and is confirmed by our experiment on MNIST with a ConvNet (see Figure~\ref{fig:divergence}). The spectral norm of the weights is multiplied by 5 over the course of 25 epochs, whereas the validation accuracy remains the same (around 99\%).
          
        Furthermore, there is an issue of vanishing gradients with BCE : first order methods struggle to saturate the logits of \LipInf networks, whereas second order methods in \textit{float64} diverge as expected. The poor properties of the optimizer, and the rounding errors in 32 bits floating point arithmetic, have greatly contributed to the caveat of BCE minimization remaining mostly unnoticed by the community.  
        
    \subsection{Lipschitz classifiers are PAC learnable}
    
        Hinge loss $\Loss^H_m$ and HKR loss $\Loss^{hkr}$ benefit from Proposition~\ref{thm:consistenceestimator}. The certificate $|f(x)|$ can be understood as confidence. Hence, we are interested in a classifier that makes a decision only if the prediction is above some threshold $m>0$, while $|f(x)|<m$ can be understood as examples $x$ for which the classifier is unsure: the label may be flipped using attacks of norm $\epsilon\leq m$. In this setting, we fall back to PAC learnability~\cite{valiant1984theory}: this theory gives bounds on the number of train samples required to guarantee that the test error will fall below some threshold $0\leq e<\frac{1}{2}$ with probability at least $1>\beta\geq 0$, through the use of Vapnik Chervonenkis (VC) dimension bounds~\cite{vapnik2015uniform}.  
          
        \begin{restatable}{proposition}{lipschitzpaclearnable}{\normalfont\textbf{1-Lipschitz Functions with margin are PAC learnable}.}\label{thm:lipschitzpaclearnable}
            Assume $P$ and $Q$ have bounded support $\Xsub$. Let $m>0$ the margin. Let $\Classifiers^m(\Xsub)=\{c^m_f:\Xsub\rightarrow\{-1,\bot,+1\},f\in\text{Lip}_1(\Xsub,\Reals)\}$ be the hypothesis class defined as follow.
            \begin{equation}
                c^m_f(x)=\begin{cases}
                        +1& \text{if } f(x)\geq m,\\
                        -1& \text{if } f(x)\leq -m,\\
                        \bot& \text{otherwise, meaning ``$f$ doesn't feel confident''.}
                \end{cases}
            \end{equation}
                Let $\Ball$ be the unit ball. Then the VC dimension of $\Classifiers^m$ is finite:
                \begin{equation}
                    (\frac{1}{m})^n\frac{vol(\Xsub)}{vol(\Ball)}\leq VC_{\text{dim}}(\Classifiers^m(\Xsub))\leq (\frac{3}{m})^n\frac{vol(\Xsub)}{vol(\Ball)}.
            \end{equation}
        \end{restatable}
          
        Interestingly if the classes are $\epsilon$ separable ($\epsilon>0$), choosing $m=\epsilon$ guarantees that 100\% accuracy is reachable. Prior over the separability of the input space is turned into VC bounds over the space of hypothesis. When $m=0$ the VC dimension of space $\Classifiers^m(\Xsub)$ becomes infinite and the class is not PAC learnable anymore: the training error will not converge to test error in general, regardless of the size of the training set. It is not a contradiction with Proposition~\ref{thm:consistenceestimator}: error $E(c^m_f(x))$ lacks continuity w.r.t $f(x)$ so it is not a consistent estimator.   
          
        This VC bound is \textit{architecture independent} which contrasts with the rest of literature on \LipInf networks. Practically, it means that the \LipCl network architecture can be chosen as big as we want without risking overfitting, as long as the margin $m$ is chosen appropriately. Proposition~\ref{thm:vcdimfullsort} also provides an architecture dependant bound for \LipCl networks. 
        \begin{restatable}{proposition}{vcdimfullsort}{\normalfont\textbf{VC dimension of \LipCl neural networks}.}\label{thm:vcdimfullsort}
            Let $f_{\theta}:\Reals^n\rightarrow\Reals$ a \LipCl neural network with parameters $\theta\in\Theta$, with \textbf{GroupSort2} activation functions, and a total of $W$ neurons. Let $\Hypothesis=\{\text{sign}f_{\theta}|\theta\in\Theta\}$ the hypothesis class spanned by this architecture. Then we have:
            \begin{equation}
                VC_{\text{dim}}(\Hypothesis)=\mathcal{O}\left((n+1)2^W\right).
            \end{equation}
        \end{restatable}
            
        In literature, tighter VC dimension bounds for neural networks exist, but they assume element-wise activation function~\cite{bartlett2019nearly}. This hypothesis does not apply to GroupSort2 which is known to be more expressive~\cite{tanielian2021approximating}, however we believe that this preliminary result can be strengthened.

      

\vspace{-2mm}
\section{Related work}
        \begin{table}
    \small
        \centering
        \begin{threeparttable}
            \begin{tabular}{cccc}
                \toprule
                 Properties & & \LipInf network & \LipCl network\\
                \midrule
                \multicolumn{2}{c}{Fit any boundary} & \yes{yes}~\cite{hassoun1995fundamentals} & \yes{yes} (Proposition~\ref{thm:lippowerbinary})\\
                \multicolumn{2}{c}{Robustness certificates} & \no{no} & \yes{yes} (Property~\ref{thm:certificates})\\
                \multicolumn{2}{c}{Consistent estimator} & \no{no} (App~\ref{app:lipinfnotrobust}) & \yes{yes} (Proposition~\ref{thm:consistenceestimator}, Figures~\ref{ex:nosameminimum},~\ref{fig:consistency})\\
                \multicolumn{2}{c}{Gradients} & exploding or vanishing & preserved for GNP (App~\ref{app:gradientpreserving})\\
                \multicolumn{2}{c}{VC dimension bounds} & architecture dependent~\cite{bartlett2019nearly} & when $m>0$ (Proposition~\ref{thm:lipschitzpaclearnable})\\
                \midrule
                \multirow{3}{*}{BCE $\Loss^{bce}_{\tau}$} & minimizer & ill-defined $L_t\veryshortarrow\infty$ (Proposition \ref{thm:saturatedlimit}) & attained (Proposition~\ref{thm:minimizerattained})\\
                & remark & vanishing gradient (Ex \ref{ex:gotcha}) & $L$ or $\tau$ must be tuned (Figure~\ref{fig:pareto_hkr})\\
                \cmidrule(r){1-4}
                \multirow{2}{*}{Wasserstein $\Loss^{W}$} & minimizer & ill-defined $L_t\veryshortarrow\infty$ & attained, robust (Property~\ref{thm:mcriswass})\\
                & remark & diverges during training & weak classifier (Proposition~\ref{thm:weakwass}) \\
                \cmidrule(r){1-4}
                \multirow{2}{*}{Hinge $\Loss^H_{m}$} & minimizer & attained & attained\\
                & remark & no guarantees on margin & $m$ must be tuned\\
                \cmidrule(r){1-4}
                \multirow{2}{*}{HKR $\Loss^{hkr}_{m,\alpha}$} & minimizer & ill-defined $L_t\veryshortarrow\infty$ & accuracy-robustness tradeoff\\
                & remark & diverges during training & $\alpha$ and $m$ must be tuned (Figure~\ref{fig:pareto_hkr})\\
                \bottomrule
            \end{tabular}
        \end{threeparttable}
        \smallskip
        \caption{
        Summary of notable results and the contributions.
        }
    \label{tab:bcesummary}\end{table}

\textbf{\LipCl networks parametrization} benefit from a rich literature (see Appendix~\ref{ap:deellip}) to enforce the Lipschitz constraint in various layers~\cite{gulrajani2017improved, yoshida2017spectral, arjovsky2017wasserstein, salimans2016weight, kim2020lipschitz, miyato2018spectral, helfrich2018orthogonal, erichson2021lipschitz, huang2021training, zhu2022controlling} such as activation functions, affine layers, attention layers or recurrent units. Residual connections are also Lipschitz (see Appendix~\ref{app:gradientpreserving}). \textbf{Gradient Norm Preserving networks} avoid the vanishing gradients~\cite{li2019preventing,bansal2018can} phenomenon to which the \LipCl networks are prone, by using orthogonal matrices in affine layers. This justifies the \textit{``orthogonal neural network''} terminology~\cite{stasiak2006fast,li2019orthogonal,su2021scaling}. ReLU based Lipschitz networks suffer from expressiveness issues~\cite{anil2019sorting}, and activation functions like GroupSort~\cite{anil2019sorting,tanielian2021approximating} (a special case of Householder reflection~\cite{mhammedi2017efficient,singla2021improved}) have been proposed in replacement. 
\textbf{Orthogonal kernels} are still an active research area~\cite{gayer2020convolutional,wang2020orthogonal,liu2021convolutional,achour2021existence,li2019preventing,trockman2021orthogonalizing,singla2021skew,yu2021constructing}. They are used in normalizing flows~\cite{hasenclever2017variational}, ensemble methods~\cite{mashhadi2021parallel}, reinforcement learning~\cite{Gogianu2021} or graph neural networks~\cite{pmlr-v139-dasoulas21a}. The optimization over the group of orthogonal matrices (known as Stiefel manifold) has been extensively studied in~\cite{absil2009optimization}, and algorithms suitable for deep learning are detailed in ~\cite{arjovsky2016unitary,hyland2017learning,hyland2017learning,lezcano2019cheap,huang2018orthogonal,ablin2021fast,kerenidis2021classical,choromanski2020stochastic}.   
  
\textbf{Generalization bounds} for general Lipschitz classifiers are given in~\cite{von2004distance,gottlieb2014efficient,bartlett2017spectrally}. Links between adversarial robustness, large margins classifiers and optimization bias are studied in~\cite{finlay2018lipschitz,jiang2018predicting,NEURIPS2018_48584348,faghri2021bridging}. The importance of the loss in adversarial robustness is studied in~\cite{pang2019rethinking}. See Appendix~\ref{app:generalizationbounds}.

\vspace{-2mm}
\section{Conclusions}

In this paper, we challenged the common belief that constraining Lipschitz constant degrades the classification performance of neural networks. We proved that \LipCl networks exhibit numerous attractive properties (see Table~\ref{tab:bcesummary} in summary): they provide robustness radius certificates without restrictions on their expressive power. They benefit from generalization guarantees. We showed that the hidden parameters of the loss allow to control the generalization gap and certifiable robustness.  
  
While the question of the \LipCl architecture is often in the spotlight, the loss is overlooked. We pointed out that Cross-Entropy is not necessarily the best choice, margin-based losses, such as hinge or its variant HKR, have appealing properties (table \ref{tab:bcesummary}). 

\vspace{-2mm}
\section{Perspectives}

This paper aims to be at the intersection between theoretical ML and (empirical) deep learning. Lipschitz constrained networks allow to directly put in perspective mathematical proofs and we are confident that this theory can be verified empirically on very large-scale vision datasets (such as Imagenet~\cite{deng09imagenet}).


This paper also provides a toolbox of results and experiments to serve as a basis for future works. We aim to open new research directions, including outside the field of robust learning. \LipInf networks could benefit from \LipCl literature: the absence of control over the Lipschitz constant of \LipInf 
is mitigated in practice by elements such as mixup or weight decay. Such elements would be better understood  by looking at how they affect the (uncontrolled) Lipschitz constant of \LipInf.

The efficient training over \LipCl is still an active research area. Moreover, \LipInf networks benefits from architectural elements such as skip connections and batch normalization (see appendix \ref{app:gradientpreserving}). As \LipCl networks get more mature, empirical results will improve, matching theory even more (explaining the emphasis on the theoretical proofs instead of the design of \LipCl depicted in appendix \ref{ap:deellip}).  
  

Many practices in deep learning entangle the questions of architecture, of generalization and of optimization. However, these elements usually have unexpected consequences on the nature of the optimum and the optimization process. Our work is a first step toward a better separation of these components and their role.
\vspace{-2mm}
\begin{ack}
We thank S\'ebastien Gerchinovitz for critical proof checking, Jean-Michel Loubes for useful discussions, and Etienne de Montbrun, Thomas Fel and Antonin Poch\'e for their read-checking. A special thank to Agustin Picard for his useful advice and thorough reading of the paper. This work has benefited from the AI Interdisciplinary Institute ANITI, which is funded by the French ``Investing for the Future – PIA3'' program under the Grant agreement ANR-19-P3IA-0004. The authors gratefully acknowledge the support of the DEEL project.\footnote{\url{https://www.deel.ai/}}  
\end{ack}

{
    \small
    \bibliographystyle{unsrt}
    \bibliography{main}

\begin{thebibliography}{10}

\bibitem{szegedy2013intriguing}
Christian Szegedy, Wojciech Zaremba, Ilya Sutskever, Joan Bruna, Dumitru Erhan,
  Ian Goodfellow, and Rob Fergus.
\newblock Intriguing properties of neural networks.
\newblock In {\em International Conference on Learning Representations}, 2014.

\bibitem{yuan2019adversarial}
Xiaoyong Yuan, Pan He, Qile Zhu, and Xiaolin Li.
\newblock Adversarial examples: Attacks and defenses for deep learning.
\newblock {\em IEEE transactions on neural networks and learning systems},
  30(9):2805--2824, 2019.

\bibitem{li2019preventing}
Qiyang Li, Saminul Haque, Cem Anil, James Lucas, Roger~B Grosse, and
  J{\"o}rn-Henrik Jacobsen.
\newblock Preventing gradient attenuation in lipschitz constrained
  convolutional networks.
\newblock In {\em Advances in Neural Information Processing Systems (NeurIPS)},
  volume~32, Cambridge, MA, 2019. MIT Press.

\bibitem{Sokolic_2017}
J.~{Sokolic}, R.~{Giryes}, G.~{Sapiro}, and M.~R.~D. {Rodrigues}.
\newblock Robust large margin deep neural networks.
\newblock {\em IEEE Transactions on Signal Processing}, 65(16):4265--4280,
  2017.

\bibitem{tsipras2019robustness}
Dimitris Tsipras, Shibani Santurkar, Logan Engstrom, Alexander Turner, and
  Aleksander Madry.
\newblock Robustness may be at odds with accuracy.
\newblock In {\em International Conference on Learning Representations}, 2019.

\bibitem{arjovsky2017wasserstein}
Martin Arjovsky, Soumith Chintala, and L{\'e}on Bottou.
\newblock Wasserstein generative adversarial networks.
\newblock In {\em International conference on machine learning}, pages
  214--223. PMLR, 2017.

\bibitem{cisse2017parseval}
Moustapha Cisse, Piotr Bojanowski, Edouard Grave, Yann Dauphin, and Nicolas
  Usunier.
\newblock Parseval networks: Improving robustness to adversarial examples.
\newblock In {\em International Conference on Machine Learning}, pages
  854--863. PMLR, 2017.

\bibitem{serrurier2020achieving}
Mathieu Serrurier, Franck Mamalet, Alberto Gonz{\'a}lez-Sanz, Thibaut Boissin,
  Jean-Michel Loubes, and Eustasio del Barrio.
\newblock Achieving robustness in classification using optimal transport with
  hinge regularization.
\newblock In {\em Proceedings of the IEEE/CVF Conference on Computer Vision and
  Pattern Recognition (CVPR)}, 2021.

\bibitem{huster2018limitations}
Todd Huster, Cho-Yu~Jason Chiang, and Ritu Chadha.
\newblock Limitations of the lipschitz constant as a defense against
  adversarial examples.
\newblock In {\em Joint European Conference on Machine Learning and Knowledge
  Discovery in Databases}, pages 16--29. Springer, 2018.

\bibitem{valiant1984theory}
Leslie~G Valiant.
\newblock A theory of the learnable.
\newblock {\em Communications of the ACM}, 27(11):1134--1142, 1984.

\bibitem{anil2019sorting}
Cem Anil, James Lucas, and Roger Grosse.
\newblock Sorting out lipschitz function approximation.
\newblock In {\em International Conference on Machine Learning}, pages
  291--301. PMLR, 2019.

\bibitem{cape2019two}
Joshua Cape, Minh Tang, and Carey~E Priebe.
\newblock The two-to-infinity norm and singular subspace geometry with
  applications to high-dimensional statistics.
\newblock {\em The Annals of Statistics}, 47(5):2405--2439, 2019.

\bibitem{hassoun1995fundamentals}
Mohamad~H Hassoun et~al.
\newblock {\em Fundamentals of artificial neural networks}.
\newblock MIT press, 1995.

\bibitem{miyato2018spectral}
Takeru Miyato, Toshiki Kataoka, Masanori Koyama, and Yuichi Yoshida.
\newblock Spectral normalization for generative adversarial networks.
\newblock In {\em International Conference on Learning Representations}, 2018.

\bibitem{bjorck1971iterative}
{\AA}ke Bj{\"o}rck and Clazett Bowie.
\newblock An iterative algorithm for computing the best estimate of an
  orthogonal matrix.
\newblock {\em SIAM Journal on Numerical Analysis}, 8(2):358--364, 1971.

\bibitem{NEURIPS2018_48584348}
Yusuke Tsuzuku, Issei Sato, and Masashi Sugiyama.
\newblock Lipschitz-margin training: Scalable certification of perturbation
  invariance for deep neural networks.
\newblock In {\em Advances in Neural Information Processing Systems},
  volume~31, pages 6541--6550. Curran Associates, Inc., 2018.

\bibitem{latorre2019lipschitz}
Fabian Latorre, Paul Rolland, and Volkan Cevher.
\newblock Lipschitz constant estimation of neural networks via sparse
  polynomial optimization.
\newblock In {\em International Conference on Learning Representations}, 2019.

\bibitem{weng2018towards}
Lily Weng, Huan Zhang, Hongge Chen, Zhao Song, Cho-Jui Hsieh, Luca Daniel,
  Duane Boning, and Inderjit Dhillon.
\newblock Towards fast computation of certified robustness for relu networks.
\newblock In {\em International Conference on Machine Learning}, pages
  5276--5285. PMLR, 2018.

\bibitem{weng2018evaluating}
Tsui-Wei Weng, Huan Zhang, Pin-Yu Chen, Jinfeng Yi, Dong Su, Yupeng Gao,
  Cho-Jui Hsieh, and Luca Daniel.
\newblock Evaluating the robustness of neural networks: An extreme value theory
  approach.
\newblock In {\em International Conference on Learning Representations}, 2018.

\bibitem{zhang2021towards}
Bohang Zhang, Tianle Cai, Zhou Lu, Di~He, and Liwei Wang.
\newblock Towards certifying l-infinity robustness using neural networks with
  l-inf-dist neurons.
\newblock In {\em International Conference on Machine Learning}, pages
  12368--12379. PMLR, 2021.

\bibitem{madry2018towards}
Aleksander Madry, Aleksandar Makelov, Ludwig Schmidt, Dimitris Tsipras, and
  Adrian Vladu.
\newblock Towards deep learning models resistant to adversarial attacks.
\newblock In {\em International Conference on Learning Representations}, 2018.

\bibitem{pmlr-v97-cohen19c}
Jeremy Cohen, Elan Rosenfeld, and Zico Kolter.
\newblock Certified adversarial robustness via randomized smoothing.
\newblock In {\em Proceedings of the 36th International Conference on Machine
  Learning}, volume~97 of {\em Proceedings of Machine Learning Research}, pages
  1310--1320. PMLR, 2019.

\bibitem{scaman2018lipschitz}
Kevin Scaman and Aladin Virmaux.
\newblock Lipschitz regularity of deep neural networks: analysis and efficient
  estimation.
\newblock In {\em Proceedings of the 32nd International Conference on Neural
  Information Processing Systems}, pages 3839--3848, 2018.

\bibitem{yang2020closer}
Yao-Yuan Yang, Cyrus Rashtchian, Hongyang Zhang, Russ~R Salakhutdinov, and
  Kamalika Chaudhuri.
\newblock A closer look at accuracy vs. robustness.
\newblock {\em Advances in Neural Information Processing Systems}, 33, 2020.

\bibitem{rousson2002shape}
Mikael Rousson and Nikos Paragios.
\newblock Shape priors for level set representations.
\newblock In {\em European Conference on Computer Vision}, pages 78--92.
  Springer, 2002.

\bibitem{chakraborty2018adversarial}
Anirban Chakraborty, Manaar Alam, Vishal Dey, Anupam Chattopadhyay, and Debdeep
  Mukhopadhyay.
\newblock A survey on adversarial attacks and defences.
\newblock {\em CAAI Transactions on Intelligence Technology}, 6(1):25--45,
  2021.

\bibitem{dong2017towards}
Yinpeng Dong, Hang Su, Jun Zhu, and Fan Bao.
\newblock Towards interpretable deep neural networks by leveraging adversarial
  examples.
\newblock {\em In AAAI-19 Workshop on Network Interpretability for Deep
  Learning}, 2017.

\bibitem{tomsett2018failure}
Richard Tomsett, Amy Widdicombe, Tianwei Xing, Supriyo Chakraborty, Simon
  Julier, Prudhvi Gurram, Raghuveer Rao, and Mani Srivastava.
\newblock Why the failure? how adversarial examples can provide insights for
  interpretable machine learning.
\newblock In {\em 2018 21st International Conference on Information Fusion
  (FUSION)}, pages 838--845. IEEE, 2018.

\bibitem{ross2018improving}
Andrew Ross and Finale Doshi-Velez.
\newblock Improving the adversarial robustness and interpretability of deep
  neural networks by regularizing their input gradients.
\newblock In {\em Proceedings of the AAAI Conference on Artificial
  Intelligence}, volume~32, 2018.

\bibitem{villani2008optimal}
C{\'e}dric Villani.
\newblock {\em Optimal transport: old and new}, volume 338.
\newblock Springer Science \& Business Media, 2008.

\bibitem{gouk2021regularisation}
Henry Gouk, Eibe Frank, Bernhard Pfahringer, and Michael~J Cree.
\newblock Regularisation of neural networks by enforcing lipschitz continuity.
\newblock {\em Machine Learning}, 110(2):393--416, 2021.

\bibitem{nar2019cross}
Kamil Nar, Orhan Ocal, S~Shankar Sastry, and Kannan Ramchandran.
\newblock Cross-entropy loss and low-rank features have responsibility for
  adversarial examples.
\newblock {\em arXiv preprint arXiv:1901.08360}, 2019.

\bibitem{bishop2006pattern}
Christopher~M Bishop.
\newblock {\em Pattern recognition and machine learning}.
\newblock springer, 2006.

\bibitem{vapnik2015uniform}
VN~Vapnik and A~Ya Chervonenkis.
\newblock On the uniform convergence of relative frequencies of events to their
  probabilities.
\newblock {\em Theory of Probability \& Its Applications}, 16(2):264--280,
  1971.

\bibitem{bartlett2019nearly}
Peter~L Bartlett, Nick Harvey, Christopher Liaw, and Abbas Mehrabian.
\newblock Nearly-tight vc-dimension and pseudodimension bounds for piecewise
  linear neural networks.
\newblock {\em Journal of Machine Learning Research}, 20:63--1, 2019.

\bibitem{tanielian2021approximating}
Ugo Tanielian and Gerard Biau.
\newblock Approximating lipschitz continuous functions with groupsort neural
  networks.
\newblock In {\em International Conference on Artificial Intelligence and
  Statistics}, pages 442--450. PMLR, 2021.

\bibitem{gulrajani2017improved}
Ishaan Gulrajani, Faruk Ahmed, Martin Arjovsky, Vincent Dumoulin, and Aaron~C
  Courville.
\newblock Improved training of wasserstein gans.
\newblock In {\em Advances in Neural Information Processing Systems},
  volume~30, pages 5767--5777. Curran Associates, Inc., 2017.

\bibitem{yoshida2017spectral}
Yuichi Yoshida and Takeru Miyato.
\newblock Spectral norm regularization for improving the generalizability of
  deep learning.
\newblock {\em arXiv preprint arXiv:1705.10941}, 2017.

\bibitem{salimans2016weight}
Tim Salimans and Diederik~P Kingma.
\newblock Weight normalization: A simple reparameterization to accelerate
  training of deep neural networks.
\newblock In {\em NIPS}, 2016.

\bibitem{kim2020lipschitz}
Hyunjik Kim, George Papamakarios, and Andriy Mnih.
\newblock The lipschitz constant of self-attention.
\newblock In {\em International Conference on Machine Learning}, pages
  5562--5571. PMLR, 2021.

\bibitem{helfrich2018orthogonal}
Kyle Helfrich, Devin Willmott, and Qiang Ye.
\newblock Orthogonal recurrent neural networks with scaled cayley transform.
\newblock In {\em International Conference on Machine Learning}, pages
  1969--1978. PMLR, 2018.

\bibitem{erichson2021lipschitz}
N.~Benjamin Erichson, Omri Azencot, Alejandro Queiruga, Liam Hodgkinson, and
  Michael~W. Mahoney.
\newblock Lipschitz recurrent neural networks.
\newblock In {\em International Conference on Learning Representations}, 2021.

\bibitem{huang2021training}
Yujia Huang, Huan Zhang, Yuanyuan Shi, J~Zico Kolter, and Anima Anandkumar.
\newblock Training certifiably robust neural networks with efficient local
  lipschitz bounds.
\newblock {\em Advances in Neural Information Processing Systems}, 34, 2021.

\bibitem{zhu2022controlling}
Zhenyu Zhu, Fabian Latorre, Grigorios Chrysos, and Volkan Cevher.
\newblock Controlling the complexity and lipschitz constant improves polynomial
  nets.
\newblock In {\em International Conference on Learning Representations}, 2022.

\bibitem{bansal2018can}
Nitin Bansal, Xiaohan Chen, and Zhangyang Wang.
\newblock Can we gain more from orthogonality regularizations in training deep
  networks?
\newblock {\em Advances in Neural Information Processing Systems}, 31, 2018.

\bibitem{stasiak2006fast}
Bart{\l}omiej Stasiak and Mykhaylo Yatsymirskyy.
\newblock Fast orthogonal neural networks.
\newblock In {\em International Conference on Artificial Intelligence and Soft
  Computing}, pages 142--149. Springer, 2006.

\bibitem{li2019orthogonal}
Shuai Li, Kui Jia, Yuxin Wen, Tongliang Liu, and Dacheng Tao.
\newblock Orthogonal deep neural networks.
\newblock {\em IEEE transactions on pattern analysis and machine intelligence},
  43(4):1352--1368, 2019.

\bibitem{su2021scaling}
Jiahao Su, Wonmin Byeon, and Furong Huang.
\newblock Scaling-up diverse orthogonal convolutional networks by a paraunitary
  framework.
\newblock In {\em Proceedings of the 39th International Conference on Machine
  Learning}, volume 162 of {\em Proceedings of Machine Learning Research},
  pages 20546--20579. PMLR, 2022.

\bibitem{mhammedi2017efficient}
Zakaria Mhammedi, Andrew Hellicar, Ashfaqur Rahman, and James Bailey.
\newblock Efficient orthogonal parametrisation of recurrent neural networks
  using householder reflections.
\newblock In {\em International Conference on Machine Learning}, pages
  2401--2409. PMLR, 2017.

\bibitem{singla2021improved}
Sahil Singla, Surbhi Singla, and Soheil Feizi.
\newblock Improved deterministic l2 robustness on cifar-10 and cifar-100.
\newblock In {\em International Conference on Learning Representations}, 2021.

\bibitem{gayer2020convolutional}
Alexander~V Gayer and Alexander~V Sheshkus.
\newblock Convolutional neural network weights regularization via
  orthogonalization.
\newblock In {\em Twelfth International Conference on Machine Vision (ICMV
  2019)}, volume 11433, page 1143326. International Society for Optics and
  Photonics, 2020.

\bibitem{wang2020orthogonal}
Jiayun Wang, Yubei Chen, Rudrasis Chakraborty, and Stella~X Yu.
\newblock Orthogonal convolutional neural networks.
\newblock In {\em Proceedings of the IEEE/CVF Conference on Computer Vision and
  Pattern Recognition}, pages 11505--11515, 2020.

\bibitem{liu2021convolutional}
Sheng Liu, Xiao Li, Yuexiang Zhai, Chong You, Zhihui Zhu, Carlos
  Fernandez-Granda, and Qing Qu.
\newblock Convolutional normalization: Improving deep convolutional network
  robustness and training.
\newblock {\em Advances in Neural Information Processing Systems},
  34:28919--28928, 2021.

\bibitem{achour2021existence}
El~Mehdi Achour, Fran{\c{c}}ois Malgouyres, and Franck Mamalet.
\newblock Existence, stability and scalability of orthogonal convolutional
  neural networks.
\newblock {\em arXiv preprint arXiv:2108.05623}, 2021.

\bibitem{trockman2021orthogonalizing}
Asher Trockman and J~Zico Kolter.
\newblock Orthogonalizing convolutional layers with the cayley transform.
\newblock In {\em International Conference on Learning Representations}, 2021.

\bibitem{singla2021skew}
Sahil Singla and Soheil Feizi.
\newblock Skew orthogonal convolutions.
\newblock In {\em International Conference on Machine Learning}, pages
  9756--9766. PMLR, 2021.

\bibitem{yu2021constructing}
Tan Yu, Jun Li, YUNFENG CAI, and Ping Li.
\newblock Constructing orthogonal convolutions in an explicit manner.
\newblock In {\em International Conference on Learning Representations}, 2021.

\bibitem{hasenclever2017variational}
Leonard Hasenclever, Jakub~M Tomczak, Rianne van~den Berg, and Max Welling.
\newblock Variational inference with orthogonal normalizing flows.
\newblock In {\em Bayesian Deep Learning, NIPS 2017 workshop}, 2017.

\bibitem{mashhadi2021parallel}
Peyman~Sheikholharam Mashhadi, S{\l}awomir Nowaczyk, and Sepideh Pashami.
\newblock Parallel orthogonal deep neural network.
\newblock {\em Neural Networks}, 140:167--183, 2021.

\bibitem{Gogianu2021}
Florin Gogianu, Tudor Berariu, Mihaela Rosca, Claudia Clopath, Lucian Busoniu,
  and Razvan Pascanu.
\newblock Spectral normalization for deep reinforcement learning: an
  optimisation perspective.
\newblock In {\em Proceedings of the International Conference on Machine
  Learning (ICML)}. JMLR. org, 2021.

\bibitem{pmlr-v139-dasoulas21a}
George Dasoulas, Kevin Scaman, and Aladin Virmaux.
\newblock Lipschitz normalization for self-attention layers with application to
  graph neural networks.
\newblock In {\em Proceedings of the 38th International Conference on Machine
  Learning}, pages 2456--2466. PMLR, 2021.

\bibitem{absil2009optimization}
P-A Absil, Robert Mahony, and Rodolphe Sepulchre.
\newblock {\em Optimization algorithms on matrix manifolds}.
\newblock Princeton University Press, 2009.

\bibitem{arjovsky2016unitary}
Martin Arjovsky, Amar Shah, and Yoshua Bengio.
\newblock Unitary evolution recurrent neural networks.
\newblock In {\em International Conference on Machine Learning}, pages
  1120--1128. PMLR, 2016.

\bibitem{hyland2017learning}
Stephanie~L Hyland and Gunnar R{\"a}tsch.
\newblock Learning unitary operators with help from u (n).
\newblock In {\em Thirty-First AAAI Conference on Artificial Intelligence},
  2017.

\bibitem{lezcano2019cheap}
Mario Lezcano-Casado and David Mart{\i}nez-Rubio.
\newblock Cheap orthogonal constraints in neural networks: A simple
  parametrization of the orthogonal and unitary group.
\newblock In {\em International Conference on Machine Learning}, pages
  3794--3803. PMLR, 2019.

\bibitem{huang2018orthogonal}
Lei Huang, Xianglong Liu, Bo~Lang, Adams Yu, Yongliang Wang, and Bo~Li.
\newblock Orthogonal weight normalization: Solution to optimization over
  multiple dependent stiefel manifolds in deep neural networks.
\newblock In {\em Proceedings of the AAAI Conference on Artificial
  Intelligence}, volume~32, 2018.

\bibitem{ablin2021fast}
Pierre Ablin and Gabriel Peyr{\'e}.
\newblock Fast and accurate optimization on the orthogonal manifold without
  retraction.
\newblock In {\em International Conference on Artificial Intelligence and
  Statistics}, pages 5636--5657. PMLR, 2022.

\bibitem{kerenidis2021classical}
Iordanis Kerenidis, Jonas Landman, and Natansh Mathur.
\newblock Classical and quantum algorithms for orthogonal neural networks.
\newblock {\em arXiv preprint arXiv:2106.07198}, 2021.

\bibitem{choromanski2020stochastic}
Krzysztof Choromanski, David Cheikhi, Jared Davis, Valerii Likhosherstov,
  Achille Nazaret, Achraf Bahamou, Xingyou Song, Mrugank Akarte, Jack
  Parker-Holder, Jacob Bergquist, et~al.
\newblock Stochastic flows and geometric optimization on the orthogonal group.
\newblock In {\em International Conference on Machine Learning}, pages
  1918--1928. PMLR, 2020.

\bibitem{von2004distance}
Ulrike von Luxburg and Olivier Bousquet.
\newblock Distance-based classification with lipschitz functions.
\newblock {\em J. Mach. Learn. Res.}, 5:669--695, 2004.

\bibitem{gottlieb2014efficient}
Lee-Ad Gottlieb, Aryeh Kontorovich, and Robert Krauthgamer.
\newblock Efficient classification for metric data.
\newblock {\em IEEE Transactions on Information Theory}, 60(9):5750--5759,
  2014.

\bibitem{bartlett2017spectrally}
Peter~L Bartlett, Dylan~J Foster, and Matus Telgarsky.
\newblock Spectrally-normalized margin bounds for neural networks.
\newblock In {\em Proceedings of the 31st International Conference on Neural
  Information Processing Systems}, pages 6241--6250, 2017.

\bibitem{finlay2018lipschitz}
Chris Finlay, Jeff Calder, Bilal Abbasi, and Adam Oberman.
\newblock Lipschitz regularized deep neural networks generalize and are
  adversarially robust.
\newblock {\em arXiv preprint arXiv:1808.09540}, 2018.

\bibitem{jiang2018predicting}
Yiding Jiang, Dilip Krishnan, Hossein Mobahi, and Samy Bengio.
\newblock Predicting the generalization gap in deep networks with margin
  distributions.
\newblock In {\em International Conference on Learning Representations}, 2019.

\bibitem{faghri2021bridging}
Fartash Faghri, Cristina Vasconcelos, David~J. Fleet, Fabian Pedregosa, and
  Nicolas~Le Roux.
\newblock Bridging the gap between adversarial robustness and optimization
  bias.
\newblock {\em ICLR 2021 Workshop on Security and Safety in Machine Learning
  Systems}, 2021.

\bibitem{pang2019rethinking}
Tianyu Pang, Kun Xu, Yinpeng Dong, Chao Du, Ning Chen, and Jun Zhu.
\newblock Rethinking softmax cross-entropy loss for adversarial robustness.
\newblock In {\em International Conference on Learning Representations}, 2019.

\bibitem{deng09imagenet}
J.~Deng, W.~Dong, R.~Socher, L.-J. Li, K.~Li, and L.~Fei-Fei.
\newblock {ImageNet: A Large-Scale Hierarchical Image Database}.
\newblock In {\em CVPR09}, 2009.

\bibitem{royden1988real}
Halsey~Lawrence Royden and Patrick Fitzpatrick.
\newblock {\em Real analysis}, volume~32.
\newblock Macmillan New York, 1988.

\bibitem{wellner1996weak}
Jon~Wellner A.W. van~der vaart.
\newblock {\em Weak Convergence and Empirical Processes}.
\newblock Springer-Verlag New York, 1996.

\bibitem{blumer1989learnability}
Anselm Blumer, Andrzej Ehrenfeucht, David Haussler, and Manfred~K Warmuth.
\newblock Learnability and the vapnik-chervonenkis dimension.
\newblock {\em Journal of the ACM (JACM)}, 36(4):929--965, 1989.

\bibitem{szarek1997metric}
Stanis{\l}aw~J Szarek.
\newblock Metric entropy of homogeneous spaces.
\newblock {\em Banach Center Publications}, 43(1):395--410, 1998.

\bibitem{vapnik2013nature}
Vladimir Vapnik.
\newblock {\em The nature of statistical learning theory}.
\newblock Springer science \& business media, 2013.

\bibitem{leon2018spaces}
Emerson Le{\'o}n and G{\"u}nter~M Ziegler.
\newblock Spaces of convex n-partitions.
\newblock In {\em New Trends in Intuitive Geometry}, pages 279--306. Springer,
  2018.

\bibitem{gers1999learning}
Felix~A Gers, J{\"u}rgen Schmidhuber, and Fred Cummins.
\newblock Learning to forget: Continual prediction with lstm.
\newblock In {\em Ninth International Conference on Artificial Neural Networks
  ICANN 99}. IET, 1999.

\bibitem{he2016deep}
Kaiming He, Xiangyu Zhang, Shaoqing Ren, and Jian Sun.
\newblock Deep residual learning for image recognition.
\newblock In {\em Proceedings of the IEEE conference on computer vision and
  pattern recognition}, pages 770--778, 2016.

\bibitem{zhang2021understanding}
Chiyuan Zhang, Samy Bengio, Moritz Hardt, Benjamin Recht, and Oriol Vinyals.
\newblock Understanding deep learning (still) requires rethinking
  generalization.
\newblock {\em Communications of the ACM}, 64(3):107--115, 2021.

\end{thebibliography}
}

\newpage
\section*{Checklist}
\begin{enumerate}

\item For all authors...
\begin{enumerate}
  \item Do the main claims made in the abstract and introduction accurately reflect the paper's contributions and scope?
    \answerYes{Proofs in appendix, experiments in the paper.}
  \item Did you describe the limitations of your work?
    \answerYes{We talked about the difficulty of training the networks.}
  \item Did you discuss any potential negative societal impacts of your work?
    \answerNA{Mostly theoretical contribution.}
  \item Have you read the ethics review guidelines and ensured that your paper conforms to them?
    \answerYes{}
\end{enumerate}

\item If you are including theoretical results...
\begin{enumerate}
  \item Did you state the full set of assumptions of all theoretical results?
    \answerYes{Assumptions are stated in each Proposition or Theorem.}
        \item Did you include complete proofs of all theoretical results?
    \answerYes{Yes (in appendix).}
\end{enumerate}

\item If you ran experiments...
\begin{enumerate}
  \item Did you include the code, data, and instructions needed to reproduce the main experimental results (either in the supplemental material or as a URL)?
    \answerYes{Experimental setting is given in appendix, and the code is in supplementary material.}
  \item Did you specify all the training details (e.g., data splits, hyperparameters, how they were chosen)?
    \answerYes{Yes, in the core paper for toy experiments and in appendix~\ref{ap:pareto} for large scale experiments.}
        \item Did you report error bars (e.g., with respect to the random seed after running experiments multiple times)?
    \answerNo{Toy experiments were run only once. CIFAR-10 experiments like in Figure\ref{fig:pareto_hkr} and Figure~\ref{fig:consistency} involved multiple runs (with different values for hyper-parameters). All the runs are reported (no cherry-picking), and the results show a coherent trend. Other experiments (see Appendix) showed our experiments are robust with respect to random seed, but discussing stability of the training of \LipCl is out of the scope of this article.}
        \item Did you include the total amount of compute and the type of resources used (e.g., type of GPUs, internal cluster, or cloud provider)?
    \answerYes{See appendix~\ref{ap:xpsetting}.}
\end{enumerate}

\item If you are using existing assets (e.g., code, data, models) or curating/releasing new assets...
\begin{enumerate}
  \item If your work uses existing assets, did you cite the creators?
    \answerYes{Yes for Deel-Lip library we used.}
  \item Did you mention the license of the assets?
    \answerYes{}
  \item Did you include any new assets either in the supplemental material or as a URL?
    \answerYes{Our code is available in supplementary.}
  \item Did you discuss whether and how consent was obtained from people whose data you're using/curating?
    \answerNA{}
  \item Did you discuss whether the data you are using/curating contains personally identifiable information or offensive content?
    \answerNA{}
\end{enumerate}

\item If you used crowdsourcing or conducted research with human subjects...
\begin{enumerate}
  \item Did you include the full text of instructions given to participants and screenshots, if applicable?
    \answerNA{}
  \item Did you describe any potential participant risks, with links to Institutional Review Board (IRB) approvals, if applicable?
    \answerNA{}
  \item Did you include the estimated hourly wage paid to participants and the total amount spent on participant compensation?
    \answerNA{}
\end{enumerate}

\end{enumerate}



\newpage
\appendix

    \tableofcontents

\section{Proofs of Section~\ref{sec:classificationpower}}\label{app:classificationpower}

    \subsection{Single class case}
    
    The properties of Lebesgue measure over euclidean space $\Reals^n$ are recalled in~\cite{royden1988real}. Importantly, the Lebesgue measure $\mu$ is translation invariant and measure the volume of hyperboxes:  $\mu([a_1,b_1],[a_2,b_2],\ldots [a_n,b_n])=\Pi_{i=1}^n(b_i-a_i)$. 
    
    \lippowerbinary*

    The proof of \textbf{Proposition~\ref{thm:lippowerbinary}} is constructive, we need to introduce the Signed Distance Function, already popularized in shape processing~\cite{rousson2002shape}.  
    
    \begin{restatable}{definition}{nlctwoclasses}{\normalfont \textbf{Signed Distance Function associated to decision boundary.}}\label{def:nlctwoclasses}
        Let $c:\Xsub\rightarrow\{-1,+1\}$ be any classifier with closed pre-images. Let $\bar{A}=\{x\in\Reals^n|c(x)=+1\}$ and $\bar{B}=\{x\in\Reals^n|c(x)=-1\}=\Xsub\setminus \bar{A}$. Let $d(x,y)=\|x-y\|$ and $d(x,S)=\min_{y\in S}d(x,y)$ be the distance to a \textbf{closed} set $S$. Let $\partial=\{x\in\Reals^n|d(x,\bar{A})=d(x,\bar{B})\}$.    
        We define $f:\Reals^n\rightarrow\Reals$ as follow:
        \begin{equation}
            f(x)=\begin{cases}
                d(x,\partial)& \text{if } d(x,\bar{B}) \geq d(x,\bar{A})\\
                -d(x,\partial)& \text{if } d(x,\bar{B}) < d(x,\bar{A}).
            \end{cases}
        \end{equation}
        We denote by SDF($c$) the function $f$.  
    \end{restatable}
    The signed distance function $f$ previously defined verifies all the properties, as a special case of Eikonal equation. We give the full proof here for completeness.  
    \begin{proof}
        We start by proving that $f$ is \mbox{1-Lipschitz}. The reasoning applies more broadly to arbitrary Banach space (topological normed vector space), not only $(\Reals^n,\|\cdot\|_2)$.  
          
        First, consider the case $d(x,\bar{B})\geq d(x,\bar{A})$ and $d(y,\bar{B})\geq d(y,\bar{A})$. Then we have $|f(x)-f(y)|=|d(x,\partial)-d(y,\partial)|$. Assume without loss of generality that $d(x,\partial)\geq d(y,\partial)$. Let $z\in\partial$ be such that $d(y,\partial)=d(y,z)$ (it is guaranteed to exist since $\partial$ is a closed set). Then by definition of $d(x,\partial)$ we have $d(x,z)\geq d(x,\partial)$. So:
        \begin{equation}
            \begin{aligned}
            |f(x)-f(y)|&=|d(x,\partial)-d(y,\partial)|\leq d(x,z)-d(y,z)\leq d(x,y).
            \end{aligned}
        \end{equation}
        The cases $d(x,\bar{B}) < d(x,\bar{A})$ and $d(y,\bar{B}) < d(y,\bar{A})$ are identical.  
        Now consider the case $d(x,\bar{B})<d(x,\bar{A})$ and $d(y,\bar{B})\geq d(y,\bar{A})$. Then we have $|f(x)-f(y)|=d(x,\partial)+d(y,\partial)$. We will proceed by contradiction. Such complicated reasoning is superfluous for $(\Reals^n,\|\cdot\|_2)$, but has the appealing property to generalize to any Banach space. Assume that $d(x,\partial)+d(y,\partial)>d(x,y)$.  
        Let $R>0$ be such that $R<d(x,\partial)$ and $R+d(y,\partial)>d(x,y)$.   
        Let:
        $$z=x+\frac{R}{d(x,y)}(x-y).$$  
        Then we have $d(x,z)=\|\frac{R}{d(x,y)}(x-y)\|=\frac{R}{d(x,y)}d(x,y)=R<d(x,\partial)$.  
        So by definition of $\partial$ we have $d(z,\bar{B})<d(z,\bar{A})$.  
        But we also have:  
        \begin{equation}
            \begin{aligned}
            d(y,z)&=\|(x-y)+\frac{R}{d(x,y)}(x-y)\|=|1-\frac{R}{d(x,y)}|\times \|x-y\|\\
                  &=|d(x,y)-R|<|d(y,\partial)|\text{ using the hypothesis on }R.
            \end{aligned}.
        \end{equation}
        So we have $d(z,\bar{B})\geq d(z,\bar{A})$ which is a contradiction. Consequently, we must have $d(x,\partial)+d(y,\partial)\leq d(x,y)$. The function $f$ is indeed \mbox{1-Lipschitz}.  
          
        Now, we will prove that $\|\nabla_xf\|=1$ everywhere it is defined. Let $x$ be such that $y\in\argmin_{y\in\partial}{d(x,y)}$ is unique. Consider $h=\epsilon\frac{(y-x)}{\|y-x\|}$ with $1\geq\epsilon>0$ a small positive real. We have $d(x,x+h)=\epsilon$, it follows by triangular inequality that $d(x+h,\partial)=d(x,\partial)-\epsilon$.   
        We see that:
        $$\lim_{\epsilon\to+\infty}\frac{f(x+h)-f(x)}{\|h\|}=-1.$$
        The vector $u=-\nabla_xf$ is the (unique) vector for which $\langle u,\frac{f(x+h)-f(x)}{\|h\|}\rangle$ is minimal. Knowing that $f$ is \mbox{1-Lipschitz} yields that $\|\nabla_xf\|=1$. For points $x$ for which $\argmin_{y\in\partial}{d(x,y)}$ is not unique, the gradient is not defined because different directions minimize $\langle u,\frac{f(x+h)-f(x)}{\|h\|}\rangle$ which contradicts the uniqueness of gradient vector. The number of points for which $y\in\argmin_{y\in\partial}{d(x,y)}$ is not unique must have null measure, since Lipschitz functions are almost everywhere differentiable (by Rademacher’s Theorem).  
          
        Finally, note that $\sign{f(x)}=c(x)$ on $\bar{A}$ and $\bar{B}$. Indeed, in this case either $d(x,\bar{B})<d(x,\bar{A})$ either $d(x,\bar{B})>d(x,\bar{A})$ and the result is straightforward.  
    \end{proof}
    
    \subsection{Multiclass case}\label{app:mutliclassexpressive}

        The label set is now $\Labels=\{1,2,\ldots,K\}$. In practice we use one-hot encoded vectors to compute the loss, by taking the $\argmax_k$ over a vector of $\Reals^K$.  
          
        \begin{proposition}[Lipschitz Multiclass classification]\label{thm:lippowermulticlass}
            For any multiclass classifier $c:\Xsub\rightarrow\Labels$ with closed pre-images there exists a 1-Lipschitz function $f:\Reals^n\rightarrow\Reals^K$ such that $\argmax_k{f_k(x)}=c(x)$ on $\Xsub$ and such that $\|\Jacobian_x f\|=1$ almost everywhere (w.r.t Lebesgue measure).  
        \end{proposition}
        
        For the case $K>2$ we must slightly change the definition to prove \textbf{Proposition~\ref{thm:lippowermulticlass}}.  
          
        \begin{definition}[Multiclass Signed Distance Function]\label{def:nlckclasses}
            Let $c:\Xsub\rightarrow\{1,2,\dots K\}$ be any classifier with closed pre-images. Let $\bar{A_k}=c^{-1}(\{k\})$. Let $\partial=\{x\in\Reals^n|\exists k\neq l, d(x,\bar{A_k})=d(x,\bar{A_l})=\argmin_m d(x,\bar{A_m})\}$.    
            We define $f:\Reals^n\rightarrow\Reals^k$ as follow:
            \begin{equation}
                f_k(x)=\begin{cases}
                    d(x,\partial)& \text{if } d(x,\bar{A_k})<d(x,\bar{A_l})\text{ for all }l\neq k,\\
                    0&\text{otherwise}.
                \end{cases}
            \end{equation}
        \end{definition}
        
        In overall the proof remains the same.  
        
        \begin{proof}
            We start by proving that $f$ is \mbox{1-Lipschitz}.  
              
            We need to prove that $\|f(x)-f(y)\|_p\leq \|x-y\|$ for any $p$-norm on $\Reals^K$ with $p\geq 1$.  
            First, consider the case $f_k(x)=f_k(y)\neq 0$. Then $\|f(x)-f(y)\|_p=|f_k(x)-f_k(y)|=|d(x,\partial)-d(y,\partial)|\leq\|x-y\|$ using the proof of Proposition~\ref{thm:lippowerbinary}.     
            Now, consider the case $f_k(x)>0,f_l(y)>0$ and $k\neq l$. Then:  
            \begin{equation}
                \begin{split}
                    \|f(x)-f(y)\|_p =\sqrt[p]{f^p_k(x)+f^p_l(y)}\leq |f^k(x)|+|f^l(y)|=d(x,\partial)+d(y,\partial).
                \end{split}
            \end{equation}
            Using the same technique as in the previous proof, if we assume $d(x,\partial)+d(y,\partial)>d(x,y)$ then we can construct $z$ verifying both $f_k(z)<f_l(z)$ and $f_l(z)<f_k(z)$ which is a contradiction. Consequently $d(x,\partial)+d(y,\partial)\leq d(x,y)$.  
              
            Each row of $\Jacobian_x f$ is either full of zeros, or the gradient of some $f_k$ on which the reasoning of the case $K=2$ applies (like in the previous proof). In this case, the spectral norm $\Jacobian_x f$ is equal to the norm of the gradient of the non zero row. We conclude similarly that $\|\Jacobian_x f\|=1$ everywhere it is defined.
              
            Finally, note that $\argmax_k f_k(x)$ is equal to $c(x)$ everywhere $c$ is defined, which concludes the proof.  
        \end{proof}  
          
        With this proposition in mind, we can deduce Corollary~\ref{thm:samedecisionfrontier}.
          
        \begin{corollary}[\LipCl is as powerful as \LipInf for classification]\label{thm:samedecisionfrontier}
        For any neural network $f:\Reals^n\rightarrow\Reals$ there exists \mbox{1-Lipschitz} neural network $\tilde{f}:\Reals^n\rightarrow\Reals$ such that $\sign{(f(x))}=\sign{(\tilde{f}(x))}$.  
        \end{corollary}  
        \begin{proof}  The proof sketched in Introduction is sufficient to show that \LipCl networks and unconstrained ones have the same decision frontiers. We could have also taken a more convoluted path: take the classifier $c$ associated to an \LipInf network, consider the restriction to a subset $\Xsub$ of the input space making the pre-images separated. Then we can apply Proposition~\ref{thm:lippowerbinary} to get a 1-Lipschitz function with the same classification power, and finally approximate those functions (in the sense of uniform convergence) with \LipCl network.  
        \end{proof}
        
        \zeroerror*
          
        \begin{proof}  If classes are separable the optimal Bayes classifier $b$ achieves zero error. Moreover, the topological closure $\closure{b^{-1}}(\{y\}),y\in\Labels$ yields a set of closed sets that are all disjoints (since $\epsilon>0$) and on which Proposition~\ref{thm:lippowermulticlass} can be applied, yielding a \LipCl neural network with the wanted properties.  
          
        \paragraph{Bonus: non separable case.} We can also handle the case of non separable classes by imitating the optimal Baye classifier $c$. We take $\Xsub$ a subset of the input space on which the pre-images of $c$ are closed.  The application of Proposition~\ref{thm:lippowerbinary} for optimal Bayes classifier gives us a 1-Lipschitz function $f$ with same decision frontier as $c$. Finally, we can use the universal approximation theorem of Anil and all. in~\cite{anil2019sorting} to conclude there exists \LipCl network that can approximate arbitrary well the function $f$, and hence approximate arbitrarily well the classifier $c$ on $\Xsub$. Outside $\Xsub$, the error is not controlled but depends of the volume of the set $(\supp{\Prob_X})/\Xsub$ whose Lebesgue measure can be made arbitrary small (by taking $\Xsub$ big enough). As $\Prob_X$ admits a pdf w.r.t Lebesgue measure, then $\Prob_X((\supp{\Prob_X})/\Xsub)$ can be made arbitrary small, and consequently the risk as well.
        \end{proof}
          
        \paragraph{Example~\ref{fig:signeddistancefunction}.} We plot the level set of the network $f$ trained from the discretized ground truth (in $400\times 400$ pixels) of the Signed distance function. The distance to the frontier $\partial$ is easily computed since the frontier $\partial$ is a finite collection of segments (fourth iteration of Von Koch snowflake fractal). We train a $128\veryshortarrow 128\veryshortarrow 128\veryshortarrow 128\veryshortarrow 128$ \LipCl network. The network is trained with Mean Square Error (MSE) and the criterion used to stop training is the Mean Absolute Error (MAE).   
    
    \subsection{Proofs of Section~\ref{sec:lipclminexists}}
    
    \minimizerattained*
    
        The proof of \textbf{Proposition~\ref{thm:minimizerattained}} is an application of Arzelà–Ascoli theorem.\begin{proof}
        Let $\Empirical(f)=\Expect_{(x,y)\sim\Prob_{XY}}[\Loss(f(x),y)]$. Consider a sequence of functions $f^t$ in $\text{Lip}_L(\Xsub,\Reals)$ such that $\lim\limits_{t\rightarrow\infty}\Empirical(f_t)=\inf_{f\in\text{Lip}_L(\Xsub,\Reals)}\Empirical(f)=\Empirical^{*}$.   
          
        Consider the sequence $u_t=\|f_t\|_{\infty}$. We want to prove that $(u_t)_{t\in\Natural}$ is bounded. Proceed by contradiction and observe that if $\limsup\limits_{t\rightarrow\infty}u_t=+\infty$ then $\limsup\limits_{t\rightarrow\infty}\Empirical(f_t)=+\infty$. Indeed, for $\|f_t\|_{\infty}\geq2L\diam{\Xsub}$ we can guarantee that $\sign{f_t}$ is constant over $\Xsub$ and in this case one of the two classes $y$ is misclassified, knowing that $\lim\limits_{f(x)\rightarrow\infty}\Loss(-yf(x),y)=\BigO(f(x))\rightarrow+\infty$ yields the desired result. But if $\limsup\limits_{t\rightarrow\infty}\Empirical(f_t)=+\infty$, then $\Empirical(f_t)$ cannot not converges to $\Empirical^{*}$. Consequently, $u_t$ must be upper bounded by some $M$.  
          
        Hence the sequence $f_t$ is uniformly bounded. Moreover each function $f_t$ is \mbox{$L$-Lipschitz} so the sequence $f_t$ is uniformly equicontinuous. By applying Arzelà–Ascoli theorem we deduce that it exists a subsequence $f_{\phi(t)}$ (where $\phi:\Natural\rightarrow\Natural$ is strictly increasing) that converges uniformly to some $f^{*}$, and $f^{*}\in\text{Lip}_L(\Xsub,\Reals)$. As $\Empirical(f^{*})=\Empirical^{*}$, the infimum is indeed a minimum.  
        \end{proof}  
        The upper bound on Lip($f$) is turned into a lower bound on $\|\nabla_{\Params}\Loss(f_L^{\Params*}(x),y)\|$ (no element-wise vanishing gradient), but its average $\|\nabla_{\Params}\Expect_{(x,y)\sim\Prob_{XY}}[\Loss((f_L^{\Params*}(x),y)]\|=0$ is null (see Appendix~\ref{app:gradientpreserving}).  
        
\section{Proofs of Section~\ref{sec:robustness}}\label{app:robustness}

        We recall below the definition of the Signed Distance Function (SDF) associated to a classifier.
        
        \nlctwoclasses*
        
        In proof of Corollary~\ref{thm:thightboundcertificate} we use the Bayes classifier $b:\Xsub\rightarrow\Labels$ associated to the classification task between $P$ and $Q$.  

        \thightboundcertificate*
    
        \begin{proof} Those properties hold by construction. The risk $\mathcal{R}(\sign{(f)})$ is minimal since $f$ is build with the optimal Bayes classifier. Note that, in general, for any classifier $c:\Xsub\rightarrow\Labels$ the bound of Property~\ref{thm:certificates} is tight by construction for SDF($c$). Indeed $f(x)$ is the distance to the frontier, and the direction is given by $\nabla_xf(x)$. 
        \end{proof}
        
        For the proof of Proposition~\ref{thm:mcriswass} we recall below the definition of Wasserstein-1 distance, as found in~\cite{villani2008optimal} (Definition 6.1).  
          
        \begin{definition}[Wasserstein-1 distance]\label{def:wasserstein}
        Let $d:\Reals^n\times\Reals^n\rightarrow\Reals$ be a metric. For any two measures $P$ and $Q$ on $\Reals^n$ the Wasserstein-1 distance is defined by the following formula:
        $$\Wasserstein_1(P,Q)\defeq\inf_{\pi\in\Pi(P,Q)}\int_{\Reals^n}d(x,y)\mathrm{d}\pi(x,y)$$
        where $\Pi(P,Q)$ denote the set of measures on $\Reals^n\times\Reals^n$ whose marginals are $P$ and $Q$ respectively. Equivalently we can write:
        $$\Wasserstein_1(P,Q)\defeq\inf_{\substack{\text{Law}(X)=P\\\text{Law}(Y)=Q}} \Expect{[d(X,Y)]}.$$
        \end{definition}
        
        In our case we are working with neural networks that are Lipschitz with respect to $l2$ distance, so we have $d(x,y) \defeq \|x-y\|_2$.  
        
        \mcriswass*
        \begin{proof} The result is straightforward by writing the dual formulation (following Kantorovich-Rubinstein) of Wasserstein $\Wasserstein_1$ metric.  
          
        By Remark 6.3 of~\cite{villani2008optimal} the Wasserstein-1 distance is the Kantorovich-Rubinstein distance:  
        $$\Wasserstein_1(P,Q)=\sup_{f\in\text{Lip}_1(\Xsub,\Reals)}\Expect_{x\sim P}[f(x)]+\Expect_{z\sim Q}[f(z)]$$
        
        We see that:
        
        \begin{equation}\label{eq:wassrobustnessproof}
            \begin{aligned}
            \Wasserstein_1(P,Q)&=\sup_{f\in\text{Lip}_1(\Xsub,\Reals)}\Expect_{x\sim P}[f(x)]-\Expect_{z\sim Q}[f(z)]\\
            &=\inf_{f\in\text{Lip}_1(\Xsub,\Reals)}\Expect_{x\sim P}[-f(x)]+\Expect_{z\sim Q}[-(-f(z))]\\
            &=\inf_{f\in\text{Lip}_1(\Xsub,\Reals)}\Expect_{(x,y)\sim\Prob_{XY}}[\Loss^W(f(x),y)].
            \end{aligned}
        \end{equation}
        
        By Kirszbraun's theorem the optimum of Equation~\ref{eq:wassrobustnessproof} can be extended into a 1-Lipschitz function over $\Reals^n$. This function can, in turn, be approximated by a \LipCl network over the domain of interest.  
        \end{proof}
        
        \weakwass*
        \begin{proof}
            We will build $P$ and $Q$ as a finite collection of Diracs. Let $P=\frac{1}{n}\sum_{i=1}^n\delta_{4(i-1)}$ and $Q=\frac{1}{n}\sum_{i=1}^n\delta_{4i-1}$ for some $n\in\Natural$, where $\delta_x$ denotes the Dirac distribution in $x\in\Reals$. A example is depicted in Figure~\ref{fig:counter_example_wass} for $n=20$. In dimension one, the optimal transportation plan is easy to compute: each atom of mass from $P$ at position $i$ is matched with the corresponding one in $Q$ to its immediate right. Consequently we must have $f(4i-1)=f(4(i-1))+3$. The function $f$ is not uniquely defined on segments $[4i-1, 4i]$ but it does not matter: since $f$ is \mbox{1-Lipschitz} we must have $|f(4i-1)-f(4i)|\leq 1$. Consequently in every case for $i<j$ we must have $f(4(i-1))<f(4(j-1))$ and $f(4i-1)<f(4j-1)$.  Said otherwise, $f$ is strictly increasing on $\supp{P}$ and $\supp{Q}$.  
  
            The solutions of the problems are invariant by translations: if $f$ is the solution, then $f - T$ with $T\in\Reals$ is also a solution. Let's take a look at classifier $c(x)=\sign{(f(x)-T)}$. If $T$ is chosen such that $f(4(i-1))-T<0$ and $f(4i-1)-T>0$ for some $1\leq i\leq n$ then $(n-1)+2=n+1$ points are correctly classified on a total of $2n$ points. It corresponds to an error of $\frac{n-1}{2n}=\frac{1}{2}-\frac{1}{2n}$. Take $n=\ceil{\frac{1}{2\epsilon}}$ to conclude.  
        \end{proof}
  
        \begin{figure}[!th]
            \centering
            \includegraphics[width=0.8\textwidth]{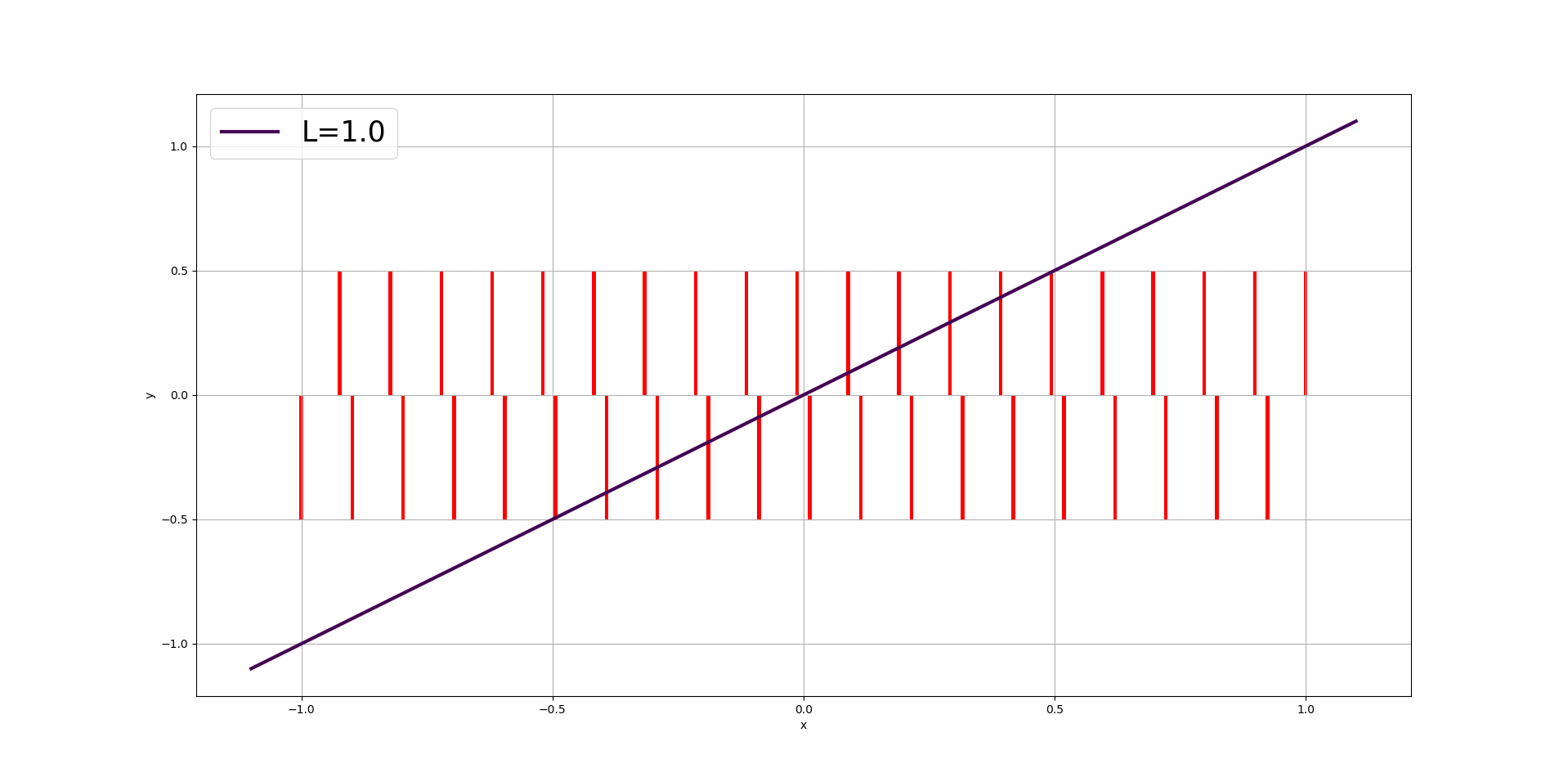}
            \vspace{-0.5cm}
            \caption{Pathological distributions $P$ and $Q$ of $20$ points each, on which the accuracy of the Wasserstein minimizer cannot be better than $52.5\%$.}
            \label{fig:counter_example_wass}
        \end{figure}

\section{Proofs of Section~\ref{sec:generalize}}\label{ap:generalize}

    \subsection{Consistency of Lipschitz estimators}
    
    \consistenceestimator*
    
        \textbf{Proof of Theorem~\ref{thm:consistenceestimator}} is an application of Glivenko-Cantelli theorem.
    
        \begin{proof}
            We proved in Proposition~\ref{thm:minimizerattained} that the minimum of equation~\ref{eq:optim1lip} is attained, so we replace $\inf$ by $\min$ for the Lipschitz loss function $\Loss$. We restrict ourselves to a subset of $\text{Lip}_L(\Xsub,\Reals)$ on which $\|f\|_{\infty}\leq2L\diam{\Xsub}$ because the minimum lies in this subspace.   
            We have:
            $$|\min_f\Empirical_p(f)-\min_f\Empirical_{\infty}(f)|\leq \max_f|\Empirical_p(f)-\Empirical_{\infty}(f)|.$$  
            Let $g_y(x)=\Loss(f(x),y)$. Note that $g$ is also Lipschitz and bounded on $\Xsub$.  
            The \textit{entropy with bracket} (see~\cite{wellner1996weak}, Chapter 2.1) of the class of functions $\mathcal{G}=\{g_y=\Loss\circ f|f\in\text{Lip}_L(\Xsub,\Reals), y\in\Labels, \Xsub\text{ bounded and }\|f\|_{\infty}\leq2L\diam{\Xsub}\}$ is finite (see~\cite{wellner1996weak}, Chapter 3.2). Consequently $\mathcal{G}$ is Glivenko-Cantelli.  
            Finally $\max_f|\Empirical_p(f)-\Empirical_{\infty}(f)|\xrightarrow{a.s}0$ which concludes the proof.  
        \end{proof}  
      
    Results of \textbf{Table~\ref{tab:bcesummary}}. Loss $\Loss^{hkr}_{m,\lambda}$ still belong to Glivenko-Cantelli classes as sum of functions $\Loss^W$ and $\Loss^H_m$ from Glivenko-Cantelli classes (on same distribution $\Prob_X$).   
      
    \subsection{Proofs of Section~\ref{sec:lipinfnotrobust}}\label{app:lipinfnotrobust}
    
    \saturatedlimit*
      
        The proof of \textbf{Proposition~\ref{thm:saturatedlimit}} only requires to take a look at the logits of two examples having different labels.  
        
        \begin{proof}
            Let $t\in\Natural$. For the pair $i,j$, as $y_i\neq y_j$, by positivity of $\Loss$ we must have:
            \begin{equation}
                0\leq\Loss(f_t(x_i),+1)+\Loss(f_t(x_j),-1)\leq\Empirical(f_t,X).
            \end{equation}
            As the right hand side has limit zero, we have:
            \begin{equation}
                \begin{split}
                    \lim\limits_{t\rightarrow\infty}\Loss(f_t(x_i),+1)&=\lim\limits_{t\rightarrow\infty}\Loss(f_t(x_j),-1)=0\\
                    \implies\lim\limits_{t\rightarrow\infty}-f_t(x_i)&=\lim\limits_{t\rightarrow\infty}f_t(x_j)=-\infty.
                \end{split}
            \end{equation}
            Consequently $\lim_{t\rightarrow\infty}|f_t(x_i)-f_t(x_j)|=+\infty$. By definition $L_t\geq \frac{|f_t(x_i)-f_t(x_j)|}{\|x_i-x_j\|}$ so $\lim_{t\rightarrow\infty}L^t=+\infty$.   
        \end{proof}  
        
        We can always find a network reaching arbitrary small loss on the train set, and arbitrary high loss on the test set. Hence, minimization of train loss does not guarantee generalization. 
          
        \begin{proposition}[\LipInf networks can always overfit]\label{thm:lipinfoverfitter}
            Assume that distributions $P$ and $Q$ admit a pdf. Let $n\in\Natural$, $M>0$ and $\epsilon>0$.  Let $(x_i,y_i)_{1\leq i\leq p}$ be a sample of $p$ iid random variables with law $\Prob_{XY}$ with $x_i\neq x_j$ for all $i\neq j$. Then there exists $f^{*}\in\LipInf$ such that:
            $$f^{*}\in\{f\in\LipInf|\Empirical_p(f)=\frac{1}{n}\sum_{i=1}^n\Loss_T(f(x_i),y_i)\leq\epsilon\}$$
            and
            $$\Empirical_{\infty}(f^{*})=\Expect_{(x,y)\sim \Prob_{XY}}[\Loss_T(f^{*}(x),y)]\geq M.$$
        \end{proposition}
        \begin{proof}
            The proof follows the strategy of Proposition~\ref{thm:saturatedlimit}. Let $d=\min_{\substack{1\leq i,j\leq n\\ i\neq j}}\|x_i-x_j\|$ the minimum distance between dataset points. We extend the dataset with a new point $(x_{n+1},y=1))$ chosen such that $\|x_j-x_{n+1}\|\geq \frac{\delta}{2}$ for all $1\leq j\leq n$. Then we transform this collections of $n+1$ Diracs functions $\sum_{i=1}^{n+1} \frac{1}{n+1}\delta_{x_i}$ into a a distribution $P$ that admits a pdf by replacing each Dirac with the uniform distribution over the ball of radius $r=\frac{d}{6}$ which yields $P=\sum_{i=1}^{n+1} \frac{1}{n+1}\mathbb{U}(\Ball(x_i,r))$. All the balls are disjoint so it exists $f\in\LipInf$ such that $\sign{f(x_i)}=y_i$ for all $1\leq i\leq n$ and $\sign{f(x_{n+1})}=-1$. Now let $|f(x_i)|\rightarrow\infty$ to guarantee that $\Empirical_p(f)\rightarrow 0$ and $\Empirical_{\infty}(f)\rightarrow\infty$. 
        \end{proof}
        
        Fortunately, as soon as the deep learning practitioner restricts itself to a subset of architectures of bounded size, the Proposition~\ref{thm:lipinfoverfitter} is no longer relevant. However, this theorem suggests that if one wants to benefit from useful generalization guarantees, one must keep the architecture of the network fixed once for all while increasing the training set size. This contradicts the trend in deep learning community to use bigger and bigger models when more data becomes available (Resnet-152, GPT3). In the light of this observation, the existence of adversarial attacks should be an expected phenomenon.  
          
        Lipschitz networks, on the other side, benefit from Proposition~\ref{thm:consistenceestimator}: minimization of train loss implies minimization of test loss. Conversely, if the test loss is high and the sample size huge, it means that the train loss is high too.

    \subsection{VC dimension for Lipschitz classifiers with margin}\label{app:marginvcbound}
    
    We recall below the definition of the Vapnik-Chervonenkis dimension~\cite{vapnik2015uniform} of a class of hypothesis, that build upon shattered sets.  
    
    \begin{definition}[\textbf{Set shattered by an hypothesis class}]\label{def:shattered}
    Let $\Labels=\{-1,+1\}$. Let $\Hypothesis$ be a class of hypothesis - that is, a set of functions $\Xsub\rightarrow\Labels$. The set of points $(x_i)_{1\leq i\leq N}\in\Xsub^N$ is said to be \textbf{shattered} by $\Hypothesis$ if for every sequence of labels $(y_i)_{1\leq i\leq N}\in\Labels^N$, there exists an hypothesis $h\in\Hypothesis$ such that for every $1\leq i\leq N$ we have $h(x_i)=y_i$.
    \end{definition}
      
    \begin{definition}[\textbf{Vapnik-Chervonenkis dimension}]\label{def:vcdim}
     The VC dimension of $\Hypothesis$, denoted $VC_{\text{dim}}(\Hypothesis)$, is the greatest integer $N\in\Natural$ such that it exists a sequence of points $(x_i)_{1\leq i\leq N}\in\Xsub^N$ shattered by $\Hypothesis$.  
    \end{definition}  
      
    Roughly speaking, the VC dimension of $\Hypothesis$ is the size of the biggest set of points such that $\Hypothesis$ agrees with any label assignment on this set of points. It measures the capacity of a set of classifiers $\Hypothesis$ to separate some sets of points. The interest of VC dimension introduced in Definition~\ref{def:vcdim} is its link with Probably Approximately Correct (PAC) learning~\cite{valiant1984theory}.  
      
    \begin{definition}[\textbf{Agnostic Probably Approximately Correct (PAC) learnability}]\label{def:pac}
    An hypothesis class $\Hypothesis$ of functions $\Xsub\rightarrow\Labels$ is PAC learnable if there exists a function $m_{\Hypothesis}:(0,1)^2\rightarrow\Natural$ and a learning algorithm $\mathcal{D}\mapsto h_m$ such that for every $(e,\beta)\in(0,1)^2$, for any distribution $\Prob_{XY}$ on $\Xsub\times\Labels$, for any dataset $\mathcal{D}=\left((x_1,y_1),(x_2,y_2),\ldots,(x_m,y_m)\right)\overset{\text{iid}}{\sim}\Prob_{XY}$ of size $m\geq m_{\Hypothesis}(e,\beta)$, we have:
    $$\Prob(E_{\Prob_{XY}}(h_m)\leq \min_{h\in\Hypothesis}E_{\Prob_{XY}}(h)+e)\geq 1-\beta$$
    We denote by $E_{\Prob_{XY}}(h)\defeq\Expect_{(x,y)\sim \Prob_{XY}}[\indicator\{h(x)\neq y\}]$ the empirical risk: the expectation of error function over $\Prob_{XY}$. 
    \end{definition}
    
    Roughly speaking, for an agnostic PAC learnable class, the probability to pick the best hypothesis $h^{*}\in\Hypothesis$ up to error $e>0$ happens with probability at least $1-\beta>0$ over datasets of size at least $m_{\Hypothesis}(e,\beta)$ sampled from distribution $\Prob_{XY}$. This definition captures the hypothesis classes that are ``small enough'' such that a reasonably high number of samples allows you to pick the best hypothesis by high probability.  
    
    The implication ``finite VC dimension'' $\implies$ ``agnostic PAC learnable'' is a classical result from \cite{blumer1989learnability}. This motivates to compute the VC dimension of Lipschitz classifiers: it yields PAC learnability results.   
    
    \lipschitzpaclearnable*
          
            \begin{proof}
                  
                This approach with margins $m$ yields objects known in the literature as $m$-fat shattering sets~\cite{gottlieb2014efficient}.  
                  
                The VC dimension of $\Classifiers^m(\Xsub)$ is the maximum size of a set shattered by $\Classifiers^m(\Xsub)$. As the functions $f$ are \mbox{1-Lipschitz}, if $c^m_f(x)=-c^m_f(y)$ then $f(x)\geq m$, $f(y)\leq m$ and $\|x-y\|\geq 2m$. Consequently, a finite set $X\subset\Xsub^n$ is shattered by $\Classifiers^m(\Xsub)$ if and only if for all $x,y\in X$ we have $\Ball(x,m)\cap\Ball(y,m)=\emptyset$ where $\Ball(x,m)$ is the open ball of center $x$ and radius $m$.  
                  
                The maximum number of disjoint balls of radius $m$ that fit inside $\Xsub$ is known as the \textbf{packing number} of $\Xsub$ with radius $m$. $\Xsub$ is bounded, hence its packing number is finite.  
                  
                The bounds on the packing number are a direct application of \cite{szarek1997metric}~(Lemma~1).
            \end{proof}  
      
    \subsection{VC dimension for GroupSort networks}\label{app:groupsortvcbound}
        
            With GroupSort2 activation functions (as in the work of~\cite{tanielian2021approximating}) we get the following rough upper bound:
              
            \vcdimfullsort*
            
            From Proposition~\ref{thm:vcdimfullsort} we can derive generalization bounds using PAC theory. Note that most results on VC dimension of neural network use the hypothesis that the activation function is applied element-wise (such as in~\cite{bartlett2019nearly}) and get asymptotically tighter lower bounds for ReLU case. Such hypothesis does not apply anymore here, however we believe that this preliminary result can be strengthened. Our result is actually a bit more general and applies more broadly to activation functions that piece-wise linear and partition the input space into convex sets.  
              
            \textbf{The proof of Proposition~\ref{thm:vcdimfullsort}} uses the number of affine pieces generated by GroupSort2 activation function, and the VC dimension of piecewise affine classifiers with convex regions.  
              
            \begin{proof}
            First, we need the following lemma.
            \begin{lemma}[Piecewise affine function]\label{thm:vcdimconvexpartition}
                    Let $\Hypothesis$ a class of classifiers that are piecewise affine, such that the pieces form a convex partition of $\Reals^n$ with $B$ pieces (each piece of the partition is a convex set). Then we have:
                    $$VC_{\text{dim}}(\Hypothesis)=\mathcal{O}\left((n+1)B^2\right).$$
                \end{lemma}
                The proof of Lemma~\ref{thm:vcdimconvexpartition} is detailed below.  
                  
                Let $\mathcal{G}(N)$ be the \textbf{growth function}~\cite{vapnik2013nature} of $\Hypothesis$. According to Sauer's lemma~\cite{vapnik2013nature} if it grows polynomially with the number of points, then the degree of the polynomial is an upper bound on the VC dimension. We will show that is indeed the case by computing a crude upper bound of the degree. Assume that we are given $N$ points, and $N$ big enough such that Sauer's lemma can be applied.  
                  
                Assume that we can choose freely the convex partition, and then only the affine classifier inside each piece. In general for neural networks that might not be the case (the boundary between partitions depends of the affine functions inside it, since neural networks are continuous); however, we are only interested in an upper bound so we can consider this generalization.  
                  
                Each piece of the partition is a polytope~\cite{leon2018spaces}. Each polytope is characterized by a set of exactly $B-1$ affine inequalities since each polytope is the intersection of $B-1$ halfspaces~\cite{leon2018spaces}. The whole partition is characterized by $\frac{B(B-1)}{2}$ affine inequalities. We divide by two because of the symmetry. Hence there exists an injective map from the set of convex partitions with $B$ pieces into $(\Reals^{n+1})^{\frac{B(B-1)}{2}}$. It is not a bijective map in general, since different systems might describe the same partition, and some degenerate systems do not correspond to partitions at all.  
                  
                We split the problem and consider each one of the $\frac{B(B-1)}{2}$ inequalities independently. According to Sauer's lemma, there is $\mathcal{O}(N^{n+1})$ ways to place the first hyperplane characterizing the first halfspace. Idem for the second hyperplane, and so on. Hence, there is at most $\mathcal{O}((N^{n+1})^{\frac{B(B-1)}{2}})$ ways to assign the $N$ points to the $B$ convex bodies.  
                  
                Each convex body (among the $B$ of them) contains atmost $N$ points, on which (still according to Sauer's lemma) there is at most $\mathcal{O}(N^{n+1})$ way to assign them labels $+1$ or $-1$, since the classifier is piecewise \textbf{affine}.    
                  
                Consequently, we have $\mathcal{G}(N)=\mathcal{O}((N^{n+1})^{\frac{B(B-1)}{2}}(N^{n+1})^B)=\mathcal{O}((N^{n+1})^{\frac{B(B+1)}{2}})=\mathcal{O}((N^{n+1})^{B^2})$. Sauer's lemma allows us to conclude:
                $$VC_{\text{dim}}(\Hypothesis)=\mathcal{O}\left((n+1)B^2\right).$$

                \textbf{Proof of the result.} Now, we need to prove that $f$ is piecewise affine and the number of such pieces is not greater than $\prod_{i=1}^k2^{\frac{w_i}{2}}=\sqrt{2^W}$, where $w_i$ is the number of neurons in layer $i$. We proceed by induction on the depth of the neural network. For depth $K=0$ we have an affine function $\Reals^n\rightarrow\Reals$ which contains only one affine piece by definition (the whole domain), so the result is true.  
                  
                Now assume that a neural network $\Reals^{w_1}\rightarrow\Reals$ of depth $K$ with widths $w_2w_3\dots w_k$ has $S_k$ affine pieces. The enumeration starting at $w_2$ is not a mistake: we pursue the induction for a neural network $\Reals^n\rightarrow\Reals$ of depth $K+1$ and widths $w_1w_2\dots w_k$. The composition of affine function is affine, hence applying an affine transformation $\Reals
                ^n\rightarrow\Reals^{w_1}$ preserves the number of pieces. The analysis falls back to the number of distinct affine pieces created by GroupSort2 activation function. If such activation function creates $S$ pieces then we have the immediate bound $S_{K+1}\leq SS_k$.  
                  
                Let $(Jf)(x)\in\Reals^{w_1\times w_1}$ be the Jacobian of the GroupSort2 operation evaluated in $x$. The cardinal $|\{(Jf)(x),x\in\Reals^{w_1}\}|$ is the number of distinct affine pieces. For GroupSort2 we have combinations of $\frac{w_i}{2}$ MinMax gates. Each MinMax gate is defined on $\Reals^2$ and contains two pieces: one on which the gate behaves like identity and the other one on which the gate behaves like a transposition. Consequently we have $S_{k+1}\leq 2^{\frac{w_k}{2}}S_k$ and unrolling the recurrence yields the desired result.  
                  
                Finally, we just need to apply the Lemma~\ref{thm:vcdimconvexpartition} with $B=\sqrt{2^W}$.

            \end{proof}  
            
    \subsection{Generalization bounds literature survey}\label{app:generalizationbounds}

    In~\cite{von2004distance} a link is established between Lipschitz classifiers and linear large-margin classifiers. Generalization bounds for large class of Lipschitz classifiers are provided by the work of~\cite{gottlieb2014efficient} using Vapnik–Chervonenkis theory. Other generalization bounds related to spectral normalization can be found in~\cite{bartlett2017spectrally}. Links between adversarial robustness, large margins classifiers and optimization bias are studied in~\cite{faghri2021bridging,finlay2018lipschitz,jiang2018predicting}. The importance of the loss in adversarial robustness is studied in~\cite{pang2019rethinking}. In~\cite{NEURIPS2018_48584348}, the control of Lipschitz constant and margins is used to guarantee robustness against attacks. A link between classification and optimal transport is established in~\cite{serrurier2020achieving} by considering a hinge regularized version of the Kantorovich-Rubinstein dual objective.

\section{Deel.Lip networks}\label{ap:deellip}

    The theorem 3 of~\cite{anil2019sorting} bound the $\|\cdot\|_{2\veryshortarrow\infty}$~\cite{cape2019two} and $\|\cdot\|_{\infty}$ norms of weight matrices to obtain universal approximation in $\text{Lip}_1(\Xsub,\Reals)$ . In practice, they reported that bounding spectral norm $\|\cdot\|_2$ and enforcing orthogonality of rows/columns of weight matrices yielded the best empirical results, because it turned the network into a \textit{Gradient Norm Preserving} network. Unfortunately, this last construction still lacks universal approximation results. Nonetheless, neither~\cite{anil2019sorting} nor ourselves encountered (so far) a function that couldn't be approximated by those GNP networks.  
      
    All the experiments done in the paper use the Deel-Lip\footnote{\url{https://github.com/deel-ai/deel-lip} distributed under MIT License (MIT).} library~\cite{serrurier2020achieving}, following ideas of~\cite{anil2019sorting}. The networks use 1) orthogonal matrices and 2) GroupSort2 activation. Orthogonalization is enforced using Spectral normalization~\cite{miyato2018spectral} and Bj{\"o}rck algorithm~\cite{bjorck1971iterative}. We have for all $i$: $$\text{GroupSort2}(x)_{2i,2i+1}=[\min{(x_{2i},x_{2i+1})},\max{(x_{2i},x_{2i+1})}].$$
    The networks parameterized by this library are GNP and belong to $\LipCl$ by construction.       
      
    The implementation of Lipschitz neural networks benefits from a rich literature. We outline below the most significant results and contributions of literature, that motivated us to use Deel-Lip library.
      
    \paragraph*{\LipCl networks parametrization.} The Lipschitz constant of affine layers can be constrained with a Gradient penalty~\cite{gulrajani2017improved} (WGAN) or spectral regularization~\cite{yoshida2017spectral}, without formal guarantee, only a very crude upper bound. Weight clipping~\cite{arjovsky2017wasserstein} (WGAN), Frobenius normalization~\cite{salimans2016weight} and spectral normalization~\cite{miyato2018spectral} lead to a tighter upper bound. However, naively stacking such layers leads to vanishing gradients. Most activation functions are Lipschitz, the popular including ReLU, sigmoid, tanh, softplus; layers such as Attention are not Lipschitz~\cite{kim2020lipschitz}. Lipschitz recurrent units have been proposed in~\cite{helfrich2018orthogonal,erichson2021lipschitz}. Residual connections are Lipschitz but prone to vanishing gradients (see Appendix~\ref{app:gradientpreserving}).  
      
    \paragraph*{Gradient Norm Preserving networks.} In \cite{gulrajani2017improved}, authors show that the potential $f$ of the Kantorovich-Rubinstein dual transport problem verifies $\|\nabla_x f(x)\|=1$ almost everywhere on the support of the distributions $\Prob_{XY}$.
    \LipCl networks fulfilling $\|\nabla_x f(x)\|=1$ almost everywhere wrt any intermediate activation $x$ are said to be Gradient Norm Preserving (GNP), and elegantly avoids the vanishing gradients phenomenon~\cite{li2019preventing,bansal2018can}. This is typically achieved in affine layers with orthogonal matrices, which justify the \textit{``orthogonal neural network''} terminology~\cite{stasiak2006fast,li2019orthogonal}. \cite{anil2019sorting}~establish that GNP networks with ReLU are exactly affine functions. They proposed Sorting activation functions to circumvent this expressiveness issue. In particular GroupSort2 revealed to be an efficient alternative~\cite{tanielian2021approximating} to ReLU, and can be seen as a particular case of Householder reflections~\cite{mhammedi2017efficient,singla2021improved}. Other authors tried to fix ReLU itself~\cite{huang2021training}.      
      
    \paragraph*{Orthogonal kernels} are of special interest in the context of normalizing flows~\cite{hasenclever2017variational}, ensemble methods~\cite{mashhadi2021parallel}, reinforcement learning~\cite{Gogianu2021} or graph neural networks~\cite{pmlr-v139-dasoulas21a}. The optimization over the orthogonal group (known as Stiefel manifold) has been extensively studied in~\cite{absil2009optimization}, while ~\cite{arjovsky2016unitary,hyland2017learning,hyland2017learning,lezcano2019cheap,huang2018orthogonal} focus on neural networks retractions like Cayley transform or exponential map; more recently~\cite{ablin2021fast} proposed a landing algorithm, ~\cite{kerenidis2021classical} proposed an algorithm inspired by quantum computing, and~\cite{choromanski2020stochastic} proposed an approach based on graph matching. Orthogonal convolutions are still an active research area: the constraint is enforced by using appropriate regularization~\cite{gayer2020convolutional,wang2020orthogonal}, by expressing convolutions in Fourrier space~\cite{liu2021convolutional,achour2021existence}, or by optimizing over the set of orthogonal convolutions directly~\cite{li2019preventing,trockman2021orthogonalizing,singla2021skew}. 
      
    In order to build CNN we used the convolution layers already provided in Deel-Lip. One limitation of these layers is that it uses the Reshaped Kernel Orthogonalization (RKO)~\cite{li2019preventing} method which, although it ensures Lipschitz bounds, does not guarantee exact orthogonality.  
      
    We also attempted to use Skew Orthogonal Convolutions (SOC, as described in algorithm~1 of~\cite{singla2021skew}). However, when we performed a sanity-check with the power iteration method, we obtained convolutions with Lipschitz constant greater than $1$.   
      
    As the method used to build 1-Lipschitz networks does not affect the conclusions of our paper, we decided to stick with Lipschitz constant provably smaller than one. We did not observe improvements by using Householder activation functions, so we used GroupSort2 activation functions instead (which are computationally cheaper).  
      
  
\section{Multiclass Hinge Kantorovich Rubinstein}\label{app:multiclasshkr}

    The loss HKR proposed by~\cite{serrurier2020achieving} was originally designed for binary classification. There are several ways to adapt it to the multi-class $K>2$ setting.   
      
    The most obvious one would be a \textit{one-versus-all} scheme. However, in multiclass classification the prediction is given by $\argmax f_k$ and not by $\sign\circ f$, so $f^{-1}(\{0\})$ is not longer the frontier. Consequently, this approach fails to yield meaningful certificates.    
      
    Instead the construction of HKR loss should once again rely on Multiclass Mean Certificate Robustness (see Definition~\ref{def:meancertifiablemulticlass}). Indeed, the robustness radius $\delta$ for class $k$ verifies:
    
    $$\|\delta\|\geq \|f(x+\delta)-f(x)\|\geq\frac{1}{2}\left( f_k(x)-\max_{i\neq k}f_i(x)\right).$$
    
    The $\frac{1}{2}$ comes from the fact that each $f_i$ is 1-Lipschitz, so their difference is $2$-Lipschitz at most. This definition is coherent with the one of \textit{multiclass hinge loss} found in most frameworks. We compare the logits of the true class with the ones of the closest other class to weight the certificate positively or negatively according to the true label.
      
    \begin{definition}[Multiclass Mean Certifiable Robustness (MMCR)]\label{def:meancertifiablemulticlass}
        For any function $f:\Xsub\rightarrow\Reals\in\LipCl$ we define its weighted multiclass mean certifiable robustness $\mathcal{R}_{(P,y)}(f)$ on class $P$ with label $k$ as:  
        \begin{equation}
            \mathcal{R}_{(P,y)}(f)\defeq\Expect_{x\sim P}[f_k(x)-\max_{i\neq k}f_i(x)].
        \end{equation} 
        Note that $f_k(x)-\argmax_{i\neq k}f_i(x)$ is either positive or negative, according to the prediction.  
    \end{definition}  
    
    Then we define the Multiclass HKR:
      
    \begin{definition}[Multiclass HKR]\label{def:multiclasshkr}
        For class an example $x$ of label $k$ let:
        $$R_k(x)\defeq f_k(x)-\argmax_{i\neq k}f_i(x).$$
        We define the multiclass HKR as:
        $$\Loss^M_{\lambda}(f(x),k)\defeq-R_k(x)+\alpha\max{(0, m-R_k(x))}.$$
    \end{definition}  
      
    For $K=2$ we recover the binary case on the function $\hat f=f_1-f_2$. Experiments showed that the \textit{one versus all} approach was outperformed by the multiclass HKR, in both robust accuracy and training time.  
      
    All $f_k$ functions can be learned independently, however in practice they share the same Lipschitz backbone and only differ in the last layer, as early experiments showed that it did not impact negatively the results. Using the same arguments as Proposition~\ref{thm:minimizerattained} based on Arzelà-Ascoli theorem we show that the minimum of $\Loss^M_{\lambda}$ is well defined and attained for each $f_k$. 
  
\section{Gradient Norm Preserving (GNP) networks}\label{app:gradientpreserving}
  
Vanishing and Exploding gradients have been a long-time issue in the training of neural networks. The latter is usually avoided by regularizing the weights of the networks and using bounded losses, while the former can be avoided using residual connections (such ideas can found on LSTM~\cite{gers1999learning} or ResNet~\cite{he2016deep}). On Gradient Norm Preserving (GNP) networks (orthogonal networks with GroupSort activation such as the ones of \textit{Deel.lip} library), we can guarantee the absence of exploding gradient:  
  
\begin{proposition}[No exploding gradients~\cite{li2019preventing}]\label{thm:noexploding}
Assume that $f = h^M\circ h^{M-1}\circ\ldots \circ h^2\circ h^1$ is a feed-forward neural network and that each layer $h^i$ is \mbox{1-Lipschitz}, where $h^i$ is either a \mbox{1-Lipschitz} affine transformation $h^i(x)=W^ix+B^i$ either a \mbox{1-Lipschitz} activation function. Let $\Loss:\Reals^k\times\Labels\rightarrow\Reals$ the loss function. Let $\tilde{y}=f(x)$, $H^i=h^i\circ h^{i-1}\circ\ldots\circ h^2\circ h^1$ and $H^0(x)=x$. Then we have:
\begin{align}
    \|\nabla_{W^i}\Loss(\tilde{y},y)\|&\leq\|\nabla_{\tilde{y}}\Loss(\tilde{y},y)\|\|H^{i-1}(x)\|\\
    \|\nabla_{B^i}\Loss(\tilde{y},y)\|&\leq\|\nabla_{\tilde{y}}\Loss(\tilde{y},y)\|.
\end{align}
\end{proposition}  
  
To prove \textbf{Proposition~\ref{thm:noexploding}} we just need to write the chain rule.
  
\begin{proof}
The gradient is computed using chain rule. Let $\theta$ be any parameter of layer $h^i$. Let $h^j_{\bot}$ be a dummy variable corresponding to the input of layer $h^j$, which is also the output of layer $h^{j-1}$. Then we have:
\begin{equation}
    \nabla_{\theta}\Loss(\tilde{y},y)=\nabla_{\tilde{y}}\Loss(\tilde{y},y)M(\Jacobian_{\theta}h^j(H^{i-1}(x))).
\end{equation}
with $M=\left(\prod_{j=M}^{i+1}\Jacobian_{h^j_{\bot}}h^j(H^{j-1}(x))\right)$.  
As the layers of the neural network are all \mbox{1-Lipschitz}, we have: $$\|\Jacobian_{h^j_{\bot}}h^j(H^{j-1}(x))\|\leq 1.$$
Hence we get the following inequality:
\begin{equation}
    \|\nabla_{\theta}\Loss(\tilde{y},y)\|\leq\|\nabla_{\tilde{y}}\Loss(\tilde{y},y)\|\|\Jacobian_{\theta}h^j(H^{i-1}(x))\|.
\end{equation}
Finally, for $h^i(H^{i-1}(x))=W^iH^{i-1}(x)+B^i$ we replace $\theta$ by the appropriate parameter which yields the desired result.
\end{proof}  

There is still a risk of vanishing gradient, which strongly depends of the loss $\Loss$. For Lipschitz neural networks, BCE $\Loss_T$ does not suffer from vanishing gradient.  
  
\begin{proposition}[No vanishing BCE gradients]\label{thm:novanishing}
Let $(x_i,y_i)_{1\leq i\leq p}$ be a non trivial training set (i.e with more than one class) such that $x_i\in\Xsub$, $\Xsub$ a \textbf{bounded} subset of $\Reals^n$. Then there exists a constant $K>0$ such that, for every minimizer $f_L^{*}$ of BCE (known to exist thanks to Proposition~\ref{thm:minimizerattained}) we have:
\begin{equation}
    f_L^{*}\in\arg\min_{f\in\text{Lip}_L(\Xsub,\Reals)}\Expect_{(x,y)\sim\Prob_{XY}}[\Loss^{bce}_T(f(x),y)].
\end{equation}
And such that for every $1\leq i\leq p$ we have the following:  
\begin{equation}
    |\frac{\partial}{\partial{\tilde{y}}}\Loss^{bce}_T(\tilde{y}=f_L^{*}(x_i),y_i)|\geq K.
\end{equation}
Note that $K$ only depends of the training set, not $f_L^{*}$.  
\end{proposition}  
\begin{proof}
  
Note that it exists $K'>0$ such that $|f_L^{*}(x_i)|\leq K'$ for all $x_i$ and all minimizers $f_L^{*}$, just like in the proof of Proposition~\ref{thm:minimizerattained}, because otherwise we could exhibit a sequence of minimizers $(f_L^{*})_t$ not uniformly bounded, which is a contradiction.  
  
Consequently $|\frac{\partial}{\partial{\tilde{y}}}\Loss^{bce}_T(\tilde{y}=f(x_i),y_i)|\geq \frac{1}{1+\exp{(|f(x_i)|)}}\geq\frac{1}{1+\exp{(K')}}=K$.  
\end{proof}  
  
It means that a non-null gradient will remain for each training example taken independently, but their mean over the train set after convergence will be the null vector. Consequently, we must expect high variance in gradients and oscillations when we get closer to the minimum.   
  
We used \textit{VGG-like} architectures instead of Resnet because GNP property makes residuals connections less useful overall (no need for shortcuts when gradient is preserved), and because those residual connections can actually be harmful: 
  
\begin{remark}[Residual connections]\label{rmk:residual}
If $f$ verifies $\|\nabla_x f(x)\|=1$ almost everywhere, and if $g$ verifies $\|\nabla_x g(x)\|=1$ almost everywhere, then $\|\nabla_x(\frac{1}{2}f(x)+\frac{1}{2}g(x))\|<1$ in general, unless $\nabla_x f(x)=\nabla_x g(x)$. Taking $f(x)=x$ we end up with residual connections, for which ensuring $\|\nabla_x(\frac{1}{2}f(x)+\frac{1}{2}g(x))\|=1$ almost everywhere is not possible unless $f=g$.  
\end{remark}

Remark~\ref{rmk:residual} essentially shows that the set of GNP layers is not stable by sum or other common operations. This makes their practical implementation and the demonstration of universal approximation theorems trickier.  
  
\section{Wasserstein discriminator does not depend of the Lipschitz constant}\label{sec:krunit}
The dual problem can be reformulated by swapping the objective and the constraint:  
\begin{equation}
    \begin{aligned}
        \arg\min_{Pf-Qf\geq\epsilon\mathcal{W}(P,Q)}\text{Lip}(f)&=\epsilon\arg\min_{Pf-Qf\geq\mathcal{W}(P,Q)}\text{Lip}(f)\\
        &=\epsilon\arg\max_{\text{Lip}(f)=1}{Pf-Qf}\\
        &=\arg\max_{\text{Lip}(f)=\epsilon}{Pf-Qf}.
\end{aligned}
\end{equation}
$\epsilon$ can be seen as re-scaling (change of units in physicist vocabulary). This makes more clear the fact that changing the Lipschitz constant is just changing the units used to measure distance. The invariance by dilation mentioned in Section~\ref{sec:robustness} must be understood in this sense: any constant $L$ can be chosen for the computation of $\Wasserstein_1$ as long as this constant is chosen in advance and bounded throughout the optimization process.

\section{BCE through the lens of OT}\label{ap:bceot}
In the following, we try to draw links between BCE minimization and optimal transport. Since the objective function is optimized with gradient descent, the gradients of the loss is the object of interest. We re-introduce $f_{\Params}$ as a function parameterized by $\Params$, mapping the input to the logits. Let $g^p_{\Params}(x)=\sigma(f_{\Params}(x))$ and $g^q_{\Params}(x)=1-\sigma(f_{\Params}(x))$. $g^p_{\Params}(x)$ (resp. $g^q_{\Params}(x)$) are the predicted probabilities of \mbox{class $+1$} (\mbox{resp. -1}). 
  
Now define $\Znc^p_{\Params}=\Expect_{x\sim P}[g^q_{\Params}(x)]$ and $\Znc^q_{\Params}=\Expect_{x\sim Q}[g^p_{\Params}(x)]$.
$\Znc^p_{\Params}$ can be seen as the weighted rate of false negatives. That is, the average mass of probability given to class $-1$ by $f_{\Params}$ when examples are sampled from class $+1$. Similarly, $\Znc^q_{\Params}$ can be seen as the rate of false positives. Let:
\begin{equation}
    \begin{alignedat}{2}
        \der P_{\Params}(x)=\frac{1}{\Znc^p_{\Params}}g^q_{\Params}(x)\der P(x)\text{ and }\der Q_{\Params}(x)=\frac{1}{\Znc^q_{\Params}}g^p_{\Params}(x)\der Q(x).
    \end{alignedat}
\end{equation}
Consequently, $P_{\Params}$ (resp. $Q_{\Params}$) is a valid probability distribution on $\Reals^n$ corresponding to the probability of an example $x$ to be incorrectly classified in class~$-1$ (resp.~$+1$). With these notations, the full expression of the gradient takes a simple form. Behold the minus sign: it is a gradient \textit{descent} and not a gradient \textit{ascent}.
\begin{equation}
        -\nabla_{\Params}\left(\Expect_{x\sim P}[\Loss(f_{\Params}(x),+1)]+\Expect_{x\sim Q}[\Loss(f_{\Params}(x),-1)]\right)=\Znc^p_{\Params}\Expect_{x\sim P_{\Params}}[\nabla_{\Params}f_{\Params}(x)]-\Znc^q_{\Params}\Expect_{x\sim Q_{\Params}}[\nabla_{\Params}f_{\Params}(x)]
\end{equation}

We apply a bias term $T\in\Reals$ to classify with $f_{\Params}-T$ instead. For a well-chosen $T$ we can enforce $\Znc^p_{\Params}=\Znc^q_{\Params}$, and such $T$ can be found using the bisection method. The optimization is performed over the set of \mbox{1-Lipschitz} functions. We end up with:
\begin{equation}
    \Znc^p_{\Params}(\Expect_{x\sim P_{\Params}}[\nabla_{\Params}f_{\Params}(x)]-\Expect_{x\sim Q_{\Params}}[\nabla_{\Params}f_{\Params}(x)]).
\end{equation}
This is the gradient for the computation of Wasserstein metric $\Wasserstein$ between $P_{\Params}$ and $Q_{\Params}$, using Rubinstein-Kantorovich dual formulation. Hence, binary cross-entropy minimization is similar to the computation of a transportation plan between errors distributions $P_{\Params}$ and $Q_{\Params}$. Note that $P_{\Params}$ and $Q_{\Params}$ depend of the current classifier $f_{\Params}-T$, so the problem is not stationary.  
  
Finally, observe that $\Loss^{bce}_{\tau}(f(x),y)=\log{2}-\frac{y\tau f(x)}{2}+\mathcal{O}(\tau^2 f^2(x))$ so when $\tau\rightarrow 0$ we get:
$$\min_{f\in\text{Lip}_1(\Xsub,\Reals)}\frac{4}{\tau}\left(\Expect_{(x,y)\sim P_{XY}}[\Loss^{bce}_{\tau}(f(x),y)]-\log{2}\right)=-\Wasserstein_1(P,Q).$$
In the limit of small temperatures, the Binary Cross-Entropy is essentially equivalent to Wasserstein. In \LipInf networks, as the training proceeds, the Lipschitz constant increases (equivalently increasing $\tau$) and the loss self-correct with $P_{\Params}$ and $Q_{\Params}$ to improve accuracy.  
  
\section{Fitting CIFAR100 with random labels}
    \label{section:cifar100_random}
    
    This experiment illustrates that constraining the Lipschitz of a network does not affect its expressive power. To show this we train a constrained network on the CIFAR100 dataset where all labels have been replaced with random labels, this task is now a widely recognized benchmark to evaluate the expressiveness of an architecture~\cite{zhang2021understanding}. 
    
    
    The architecture of this network is as simple as possible: two orthogonal dense layers with 1024 neurons are followed by a dense layer which is normalized but not orthogonal. 
    The GroupSort2 activation function is used and biases are enabled.
    
    Hyper parameters for this experiment are listed in table \ref{fig:cifar100-lp}, and results are reported in table \ref{fig:cifar100-results}.
    
    \begin{figure}
        \centering
        \begin{subfigure}[b]{0.30\textwidth}
            \small
            \centering
            \begin{tabular}{cc}
                \hline
                parameter & value \\
                \hline
                data augmentation & none \\
                input scale & $[0,1]$ \\
                batch size & 1000 \\
                learning rate & 0.001 \\
                optimizer & Adam \\
                cosine decay & 0.01 \\
                architecture & 1024-1024-100 \\
                activation & GroupSort2 \\
                epochs &  250 \\
                \hline
            \end{tabular}
            \caption{Learning parameters on CIFAR-100 with random labels.}
            \label{fig:cifar100-lp}
        \end{subfigure}
        \hfill
        \begin{subfigure}[b]{0.60\textwidth}
            \small
            \centering
            \begin{tabular}{cp{1.5cm}p{1.5cm}}
                \hline
                loss & $CCE$ \newline $\tau=256$ & $HKR$ \newline $\alpha=256$ \newline $m=\frac{36}{255}$ \\
                \hline
                accuracy & 0.999 & 0.998 \\
                robust accuracy $\epsilon=36$ & 0.382 & 0.91 \\
                robust accuracy $\epsilon=72$ & 0.021 & 0.19 \\
                lipschitz upper bound &  1.002 & 1.002 \\
                \hline
            \end{tabular}
            \caption{Learning results for CIFAR-100 with random labels: validation accuracy is not reported as its value is always 0.01 for clean accuracy and 0.00 for robust accuracy. The Lipschitz upper bound is computed using the power iteration method on each layer.}
            \label{fig:cifar100-results}
        \end{subfigure}
        \caption{}
    \end{figure}

    At first glance it's might seem surprising to see both high accuracy and high provable robustness on a dataset with random labels. This is compliant with the idea expresses by the authors of \cite{yang2020closer}: for a given accuracy one can increase the robustness radius around a sample $x_1$ up to the value $\| \frac{x_1 -x_2}{2} \| $ where $x_2$ is the closest sample with a different label. The decision frontier is close to the decision frontier of the 1-nearest neighbor based on the trained set.
    This illustrates that constraining the Lipschitz constant does not necessarily decrease accuracy and does not necessarily increase robustness. Also, it shows that there is no trivial link between robustness and generalization.

\section{1-Lipschitz estimators are consistent (experimental protocol of figure \ref{fig:consistency})}\label{ap:consistency}

    Lipschitz classifiers are consistent: as the size of the training set increases, the training loss becomes a proxy for the test loss. However, we do not give convergence speed bounds: we do not know how many samples are needed for a given task to observe the convergence between train and test losses. Moreover, the losses are parametrized (e.g by $\tau,\alpha,m$) so we expect to have different convergence rates, depending on those parameters. In order to observe this empirically on the CIFAR10 dataset, the same architecture (described in fig \ref{fig:consistency-archi}) was trained successively on 2\%, 5\%, 10\%, 25\%, 50\% and 100\% of the dataset. The sub-sampling was performed with a different seed each time, showing that, for this range of $\tau$ training is still stable. Similarly, this procedure has been repeated with different values for $\tau$. Hyper-parameters used for learning are reported in fig \ref{fig:consistency-lp}. Learning results are reported in Figure~\ref{fig:consitency-results}. It also shows certifiable accuracy and empirical accuracy, that were not displayed in the fig~\ref{fig:consistency} from the main paper. We see that lower values for $\tau$ yield tighter robustness certificates (certificate value is close to the distance found by L2-PGD).

    \begin{figure}
        \small
        \centering
        \begin{tabular}{llp{1.2cm}p{1.2cm}p{1.2cm}p{1.2cm}p{1.2cm}p{1.2cm}p{1.2cm}}
            \hline
            $\tau$ & \% dataset & accuracy & certifiable\newline $\epsilon:36$ & L2-PGD \newline $\epsilon:36$ &  certifiable \newline $\epsilon:72$ & L2-PGD \newline $\epsilon:72$ & certifiable \newline $\epsilon:108$ & L2-PGD \newline $\epsilon:108$ \\
            \hline
               8.0 &         1 &   0.6279 &          0.3075 &      0.525 &          0.0996 &      0.462 &           0.0204 &       0.388 \\
               4.0 &         1 &   0.6207 &          0.4289 &      0.568 &          0.2553 &      0.502 &           0.1308 &       0.443 \\
               2.0 &         1 &   0.5683 &          0.4456 &      0.535 &          0.3298 &      0.496 &           0.2313 &       0.457 \\
               1.0 &         1 &   0.5097 &          0.4331 &      0.482 &          0.3615 &      0.436 &           0.2926 &       0.404 \\
               0.5 &         1 &   0.4497 &           0.398 &      0.434 &          0.3478 &        0.4 &           0.3035 &       0.369 \\
              0.25 &         1 &   0.4064 &          0.3703 &      0.395 &          0.3346 &      0.373 &           0.2974 &       0.351 \\
               8.0 &       0.5 &   0.5959 &          0.2862 &      0.526 &          0.0973 &      0.443 &           0.0218 &       0.363 \\
               4.0 &       0.5 &   0.5884 &          0.3967 &      0.511 &          0.2275 &       0.44 &           0.1123 &       0.383 \\
               2.0 &       0.5 &   0.5569 &          0.4393 &      0.503 &          0.3298 &      0.463 &           0.2277 &       0.419 \\
               1.0 &       0.5 &   0.5012 &          0.4235 &      0.443 &          0.3525 &      0.403 &           0.2875 &       0.366 \\
               0.5 &       0.5 &   0.4486 &          0.3996 &      0.419 &          0.3517 &      0.397 &           0.3093 &       0.367 \\
              0.25 &       0.5 &   0.3928 &          0.3574 &      0.373 &          0.3224 &      0.352 &           0.2969 &       0.331 \\
               8.0 &      0.25 &   0.5553 &          0.2801 &      0.484 &          0.0987 &      0.402 &           0.0252 &       0.322 \\
               4.0 &      0.25 &   0.5696 &          0.3835 &      0.493 &          0.2208 &      0.437 &            0.112 &       0.382 \\
               2.0 &      0.25 &   0.5397 &          0.4156 &      0.482 &           0.299 &      0.431 &           0.2059 &       0.384 \\
               1.0 &      0.25 &   0.5013 &          0.4205 &      0.455 &           0.345 &      0.423 &           0.2779 &       0.389 \\
               0.5 &      0.25 &   0.4448 &          0.3919 &      0.418 &          0.3449 &      0.392 &           0.2983 &       0.372 \\
              0.25 &      0.25 &   0.3939 &          0.3593 &      0.359 &          0.3278 &      0.348 &           0.2944 &       0.327 \\
               8.0 &       0.1 &   0.4914 &          0.2617 &      0.396 &          0.1133 &      0.335 &           0.0396 &       0.271 \\
               4.0 &       0.1 &   0.5053 &          0.3397 &      0.449 &          0.2001 &      0.378 &           0.1024 &       0.322 \\
               2.0 &       0.1 &    0.503 &          0.3858 &      0.478 &          0.2793 &      0.426 &           0.1898 &       0.387 \\
               1.0 &       0.1 &   0.4783 &          0.3977 &      0.428 &          0.3188 &      0.391 &           0.2484 &       0.355 \\
               0.5 &       0.1 &   0.4385 &          0.3824 &      0.425 &           0.331 &      0.401 &           0.2799 &       0.383 \\
              0.25 &       0.1 &   0.3872 &          0.3517 &      0.379 &          0.3168 &       0.35 &           0.2874 &       0.328 \\
               8.0 &      0.05 &    0.445 &           0.266 &      0.406 &          0.1324 &      0.337 &           0.0585 &       0.264 \\
               4.0 &      0.05 &   0.4508 &          0.3052 &      0.439 &           0.187 &      0.374 &           0.1078 &       0.309 \\
               2.0 &      0.05 &   0.4562 &          0.3444 &      0.395 &          0.2478 &      0.342 &           0.1723 &       0.293 \\
               1.0 &      0.05 &   0.4392 &          0.3579 &      0.416 &          0.2898 &      0.377 &           0.2272 &       0.343 \\
               0.5 &      0.05 &   0.4176 &          0.3629 &      0.396 &          0.3106 &      0.365 &           0.2679 &       0.326 \\
              0.25 &      0.05 &   0.3741 &          0.3379 &      0.352 &          0.3028 &      0.323 &           0.2707 &       0.306 \\
               8.0 &      0.02 &   0.3778 &          0.2539 &      0.327 &           0.159 &       0.27 &           0.0922 &       0.227 \\
               4.0 &      0.02 &   0.3761 &          0.2738 &      0.298 &          0.1879 &      0.254 &           0.1236 &       0.207 \\
               2.0 &      0.02 &   0.3859 &          0.2965 &      0.325 &          0.2167 &      0.278 &           0.1549 &       0.236 \\
               1.0 &      0.02 &    0.385 &          0.3129 &       0.35 &          0.2522 &      0.312 &           0.1999 &       0.273 \\
              0.25 &      0.02 &   0.3378 &          0.3021 &      0.305 &          0.2684 &      0.288 &           0.2375 &        0.26 \\
               0.5 &      0.02 &   0.3631 &          0.3091 &      0.322 &          0.2605 &      0.302 &           0.2172 &       0.282 \\
            \hline
        \end{tabular}
        \caption{Network trained on different fractions of the CIFAR-10 dataset. For each value of $\tau$ and each dataset fraction, clean accuracy, certifiable and empirical accuracies are reported. We report accuracy under l2-PGD attack to perform a sanity check of the network's certificates. Interestingly, a lower temperature leads to a tighter bound for certifiable robustness (a lower gap between certifiable robustness and empirical robustness).}
        \label{fig:consitency-results}
    \end{figure}
    
    \begin{figure}
    \begin{subfigure}[b]{0.49\textwidth}
            \small
            \centering
            \begin{tabular}{c}
                \hline
                network architecture \\
                \hline
                conv-3x3-32 (groupsort 2)\\
                conv-3x3-32 (groupsort 2) \\
                invertible downsampling \\
                conv-3x3-64 (groupsort 2) \\
                conv-3x3-64 (groupsort 2) \\ 
                invertible downsampling \\
                conv-3x3-128 (groupsort 2) \\
                conv-3x3-128 (groupsort 2) \\ 
                flatten \\
                dense-128 (groupsort 2) \\
                dense-101 (None) \\
                \hline
            \end{tabular}
            \caption{Network architecture used in the consistency experiment. It has 1.6M trainable parameters.}
            \label{fig:consistency-archi}
        \end{subfigure}
        \hfill
        \begin{subfigure}[b]{0.49\textwidth}
            \small
            \centering
            \begin{tabular}{cc}
                \hline
                parameter & value \\
                \hline
                data augmentation & None \\
                input scale & $[0,1]$ \\
                batch size & 1000 \\
                learning rate & 1e-5 \\
                optimizer & Adam \\
                cosine decay & None \\
                epochs &  300 \\
                \hline
            \end{tabular}
            \caption{Training parameters used in the consistency experiment. No data augmentation has been used, as it would artificially increase the number of samples in the dataset and biases the results.}
            \label{fig:consistency-lp}
        \end{subfigure}
        \caption{}
    \end{figure}

\section{Controlling the accuracy/robustness tradeoff (experimental protocol  of figure \ref{fig:pareto_hkr})}\label{ap:pareto}

    
    A small CNN architecture, described more precisely in Figure~\ref{fig:pareto-archi} was trained multiple times with different losses and different loss parameters. Besides the learning rate, other parameters were left unchanged, and are depicted in Figure~\ref{fig:pareto-lp}. The learning rate is chosen depending on the loss parameters: when trained with CCE, changing $\tau$ implicitly changes the norm of the gradient, thus when doubling $\tau$ one must divide the learning rate by a factor of two. The same phenomenon occurs with the $\alpha$ parameter of HKR. We kept $m=20$ fixed for HKR, and we tuned $\alpha$.  
    
    For each run the final validation accuracy is reported on \textbf{x}-axis while the provable accuracy at $\epsilon=36$ is reported on \textbf{y} axis. The choice of the robustness metric was set to the robust accuracy at $\epsilon=36$ because of its wide use in the community. However,  other robustness metrics also yield a Pareto front: two examples are shown in Figures~\ref{fig:pareto-avg} and~\ref{fig:pareto-MCR}. In those examples, we also test different combinations of values for $m$ in HKR.  
    
    \begin{figure}
        \begin{subfigure}[b]{\textwidth}
                \small
                \centering
                \includegraphics[width=\textwidth]{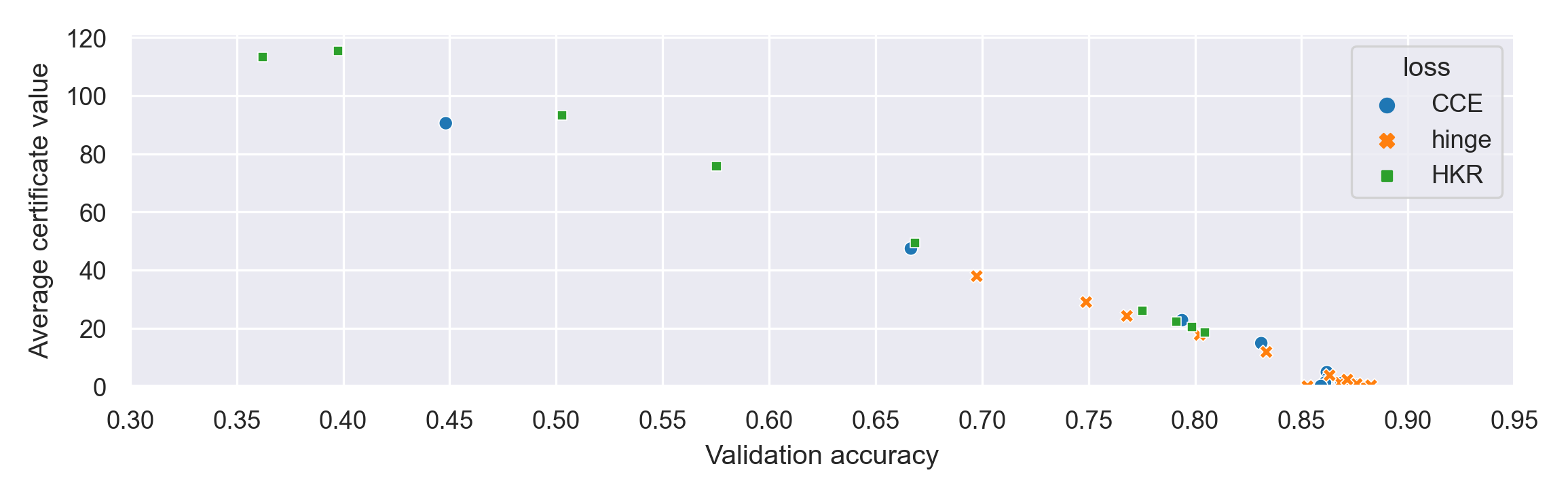}
                \caption{}
                \label{fig:pareto-avg}
        \end{subfigure}
        \hfill
        \begin{subfigure}[b]{\textwidth}
            \small
            \centering
            \includegraphics[width=\textwidth]{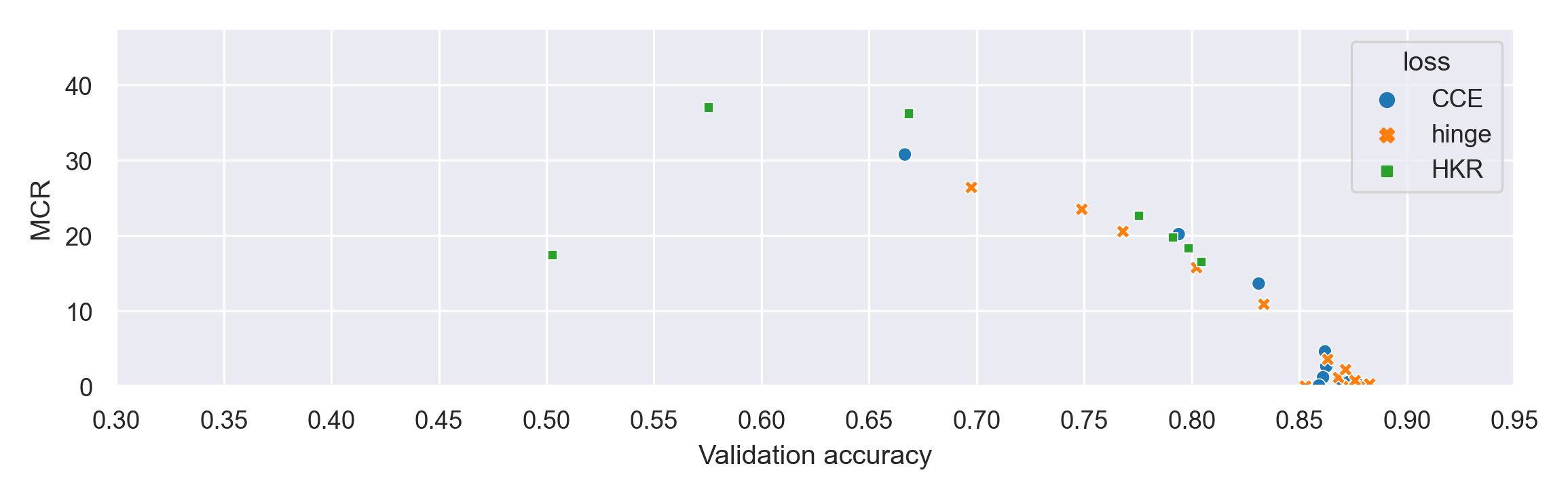}
            \caption{}
            \label{fig:pareto-MCR}
        \end{subfigure}
        \caption{Pareto front for other robustness metrics: depending on the metric chosen to evaluate robustness, the shape of the Pareto front is changed. Upper chart shows use the average certificate value (robustness that does \textbf{not} take into account the true label, only the average value of $|f(x)|$), while the lower uses the MCR. The same models are used for these two graphs and Fig~\ref{fig:pareto_hkr}.}
    \end{figure}
    
    Note that this Pareto front can also be influenced by other factors: training larger architectures could improve both accuracy and robustness. Similarly, data augmentation has an impact on accuracy and robustness, but studying the phenomenon is out of the scope of this paper.
    
    Finally, it is important to note that, the comparison between two architectures (or two robustness methods) cannot be done properly with a single training (and fixed hyper-parameters): comparing their Pareto front is more relevant.
    
    \begin{figure}
        \begin{subfigure}[b]{0.35\textwidth}
                \small
                \centering
                \begin{tabular}{c}
                    \hline
                    network architecture \\
                    \hline
                    conv-3x3-32 (groupsort 2)\\
                    conv-3x3-32 (groupsort 2) \\
                    L2 norm pooling 2D \\
                    conv-3x3-64 (groupsort 2) \\
                    conv-3x3-64 (groupsort 2) \\ 
                    conv-3x3-64 (groupsort 2) \\ 
                    L2 norm pooling 2D \\
                    conv-3x3-128 (groupsort 2) \\
                    conv-3x3-128 (groupsort 2) \\
                    conv-3x3-128 (groupsort 2) \\ 
                    global L2 norm pooling 2D \\
                    dense-128 (groupsort 2) \\
                    dense-101 (None) \\
                    \hline
                \end{tabular}
                \caption{Network architecture used in the Pareto front experiment. It has 0.4M trainable parameters.}
                \label{fig:pareto-archi}
            \end{subfigure}
            \hfill
            \begin{subfigure}[b]{0.6\textwidth}
                \small
                \centering
                \begin{tabular}{cp{5.5cm}}
                    \hline
                    parameter & value \\
                    \hline
                    data augmentation & \begin{itemize}[leftmargin=*]
                        \item random flip left right
                        \item random brightness: $\delta=0.2$
                        \item random contrast: lower=0.75, upper=1.3
                        \item random hue: $\delta=0.1$
                        \item random saturation: lower=0.8 upper=1.2
                        \item random crop: $\text{scale}\in[0.8, 1.0]$
                    \end{itemize} \\
                    input scale & $[0,256]$ \\
                    batch size & 512 \\
                    learning rate & $[5\times10^{-2}, 1\times10^{-2}, 5\times10^{-3}, 1\times10^{-3}]$ \\
                    optimizer & Adam \\
                    cosine decay & 1e-2 \\
                    epochs &  300 \\
                    \hline
                \end{tabular}
                \caption{Training parameters used to build the Pareto front between accuracy and robustness. As the loss parameters implicitly change gradient norm, learning rate has been changed adequately, ranging from $5e-2$ (low $\tau$ and low $\alpha$) to $1e-3$ (high $\tau$ and high $\alpha$).}
                \label{fig:pareto-lp}
            \end{subfigure}
            \caption{}
    \end{figure}

\section{Hardware}\label{ap:xpsetting}
  
Toy experiments depicted in Fig.~\ref{ex:nosameminimum}, Fig.~\ref{fig:signeddistancefunction}, Fig.~\ref{fig:twomoonstradeoff} and example~\ref{ex:gotcha}  were run on a personal workstation with NVIDIA Geforce 1080 GTX and 8GB VRAM, 16 cores Xeon and 32GB RAM.  
  
Large scales experiments depicted in Fig.~\ref{fig:consistency}, Fig.\ref{fig:pareto_hkr} were run on Google Cloud with TPU v2-8. For reference, the experiments with CNNs on CIFAR10 (appendix  \ref{ap:pareto}), took 4.9s per epoch on average.

\textit{Tensorflow} framework was used in every experiment but the one of Example~\ref{ex:gotcha}, where \textit{Jax} was used instead (because {order2} and \textit{float64} experiments are easier to write in this library).  
  
\section{Divergence of the weights on \LipInf networks}\label{app:divergence}

In this example, we illustrate that example\ref{ex:gotcha} behavior can be observed at larger scale on MNIST with a ConvNet of \LipInf. We used $3\times 3$ convolution filters of widths $32\veryshortarrow 64$ with \textbf{MaxPool} and \textbf{ReLU}, followed by a flattening operation and densely connected layers of widths $256\veryshortarrow 10$.  
  
Newton's method cannot be used due to its memory requirements on ConvNet. We tested SGD with learning rate $\eta=0.1$ and momentum $m=0.9$, and Adam with learning rate $\eta=1e-3$ and other default parameters. Experiments were run both in \textit{float32} and \textit{float64} precision. We monitor the maximum spectral norm of the weights of the network throughout training for each epoch $t\in\Natural$:
$$\mathcal{M}^t=\max_i\|W^t_i\|_2.$$
  
We report $\mathcal{M}^t$ as function of epoch $t$ in Figure~\ref{fig:divergence}. The validation accuracy is above 98\% after the first epoch, and fluctuates between 98.5\% and 99.5\% during the following epochs (in either cases). Similarly the validation loss fluctuates between $1e-1$ and $1e-3$. We see that on this simple task the spectral norm of weight matrices continues to grow indefinitely, even though the classifier is almost perfect after one epoch. Interestingly, on this experiment the vanishing gradient phenomenon cannot be observed after 25 epochs and the results are robust with respect to the precision of the floating point arithmetic.     
  
This is compliant with the observations made in the literature about the high Lipschitz constant of \LipInf networks~\cite{scaman2018lipschitz}. We observe that Adam makes the problem worse, even if its learning rate is smaller. This may explain why many practitioners reported that Adam was more susceptible to overfit than SGD with a carefully tuned learning rate scheduling.   
  
\begin{figure}
    \centering
    \includegraphics[scale=0.8]{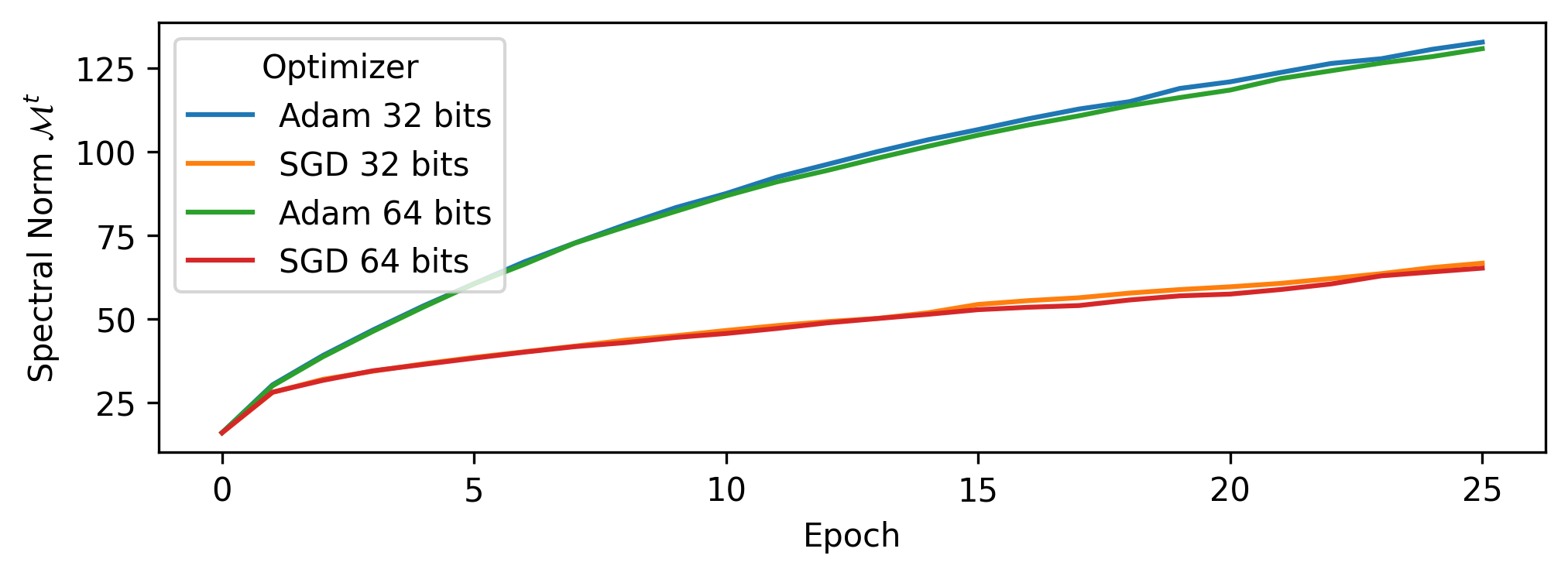}
    \caption{Maximum spectral norm of the weights of a simple ConvNet of \LipInf trained with different optimizers on MNIST dataset. The validation accuracy remains above $98.5\%$ after the second epoch but the network's weights do not converge: the spectral norm seems to grow indefinitely.}
    \label{fig:divergence}
\end{figure}

\section{Stability of training of \LipCl}\label{app:stability}

To check this, we trained different \LipCl networks, either by tuning the value of temperature $\tau$, or by tuning the number of filters in convolutional layers. We used Fashion-Mnist dataset.  

\subsection{Moving along Pareto front by tuning temperature}

\begin{wrapfigure}{R}{0.65\textwidth}
  \includegraphics[width=0.65\textwidth]{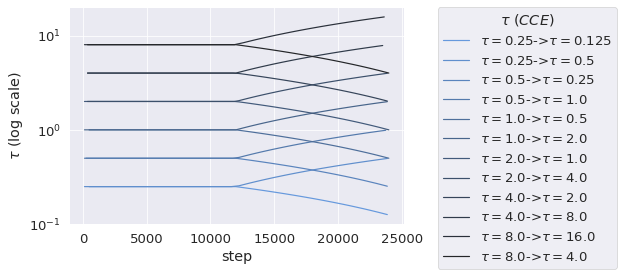}
  \caption{
    Schedule for the $\tau$ parameter: $\tau$ is set to be constant for 200 epochs, followed by a linear increased/decreased by a factor of 2 for 200 epochs. 
  }
  \label{fig:schedule}
\end{wrapfigure}

In this experiment, we explore the stability of training with respect to the loss parameter (here we use CCE with $\tau$). To do so we perform a scheduling on $\tau$: each network of the experiment is trained with a fixed $\tau$ for 200 epochs, then $\tau$ is increased/decreased linearly by a factor of 2 for 200 more epochs (see \ref{fig:schedule}.
We train a total of 12 \LipCl networks with the same architecture: three blocks of two convolutions, bias and group sort are followed by a Pooling layer (L2NormPooling) finally followed by a flatten and a Dense layer (architecture synthesized as c32-c32-P-c64-c64-P-c128-c128-D256 ). Each training is performed with the same optimizer (Adam) and the same learning rate (0.001). Each dot in the graph of Figure~\ref{fig:tau_curriculum} corresponds to the metrics of a network after one epoch. The validation accuracy can be found on \textbf{x-axis} and MCR metric on \textbf{y-axis}.
  
During the first epochs, the dots can be found inside the region delimited by the Pareto front: the network has not converged yet, and both Mean Certifiable Robustness and validation accuracy are low. After few epochs, the dots start to accumulate on the Pareto front. Then, the value of temperature $\tau$ is tuned \textit{during} the training, from initial $\tau_{\text{start}}$ to $\tau_{\text{final}}$. Each of the color corresponds to a different value of Tau. We see that the temperature can be modified during training to move along the Pareto front.

\begin{figure}
    \centering
    \includegraphics[width=\textwidth]{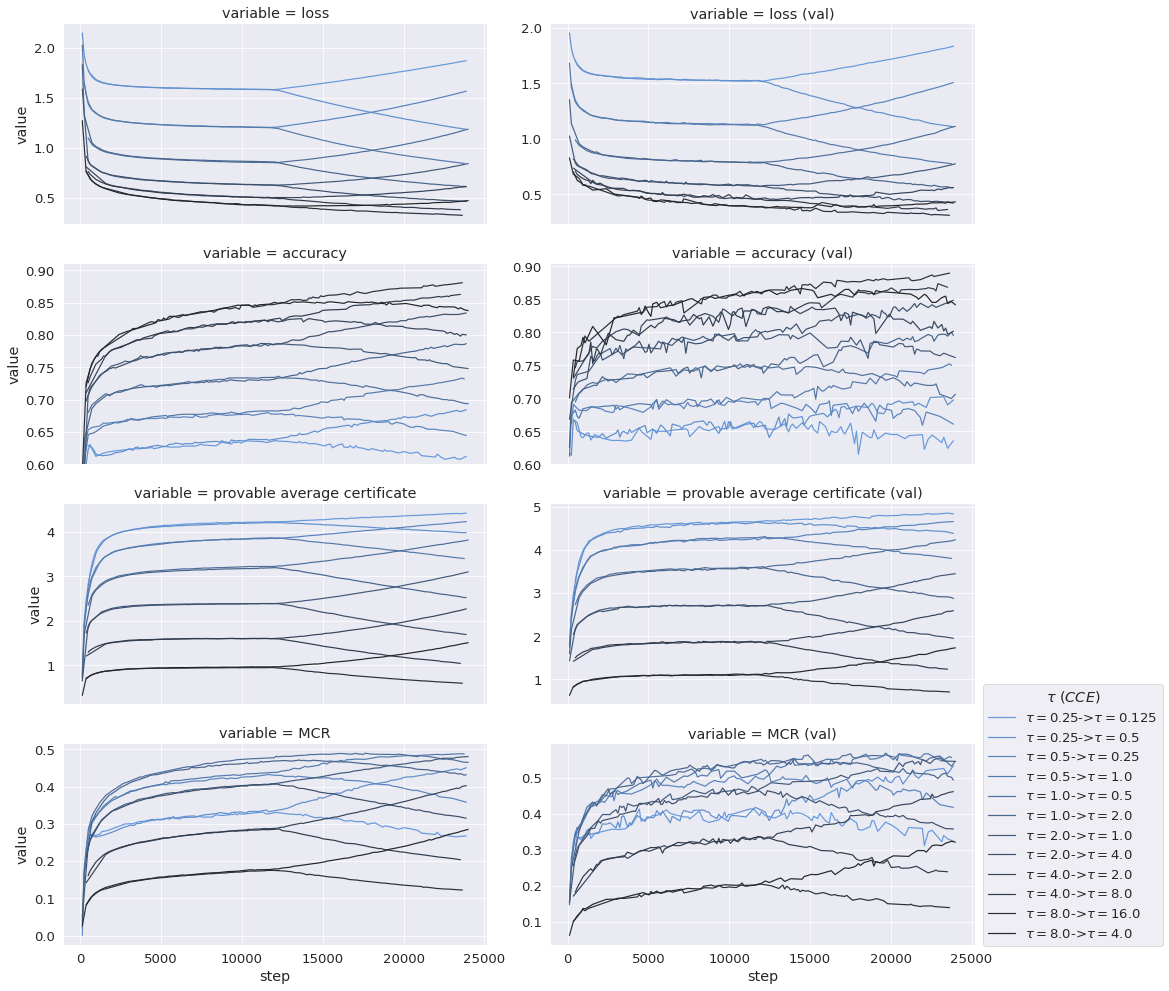}
    \caption{\textbf{Learning curves on Fashion-Mnist}: 12 \LipCl networks are trained using the schedule depicted in fig \ref{fig:schedule} with respectively $\tau_{\text{init}}\in [0.25, 0.5, 1, 2, 4, 8]$. At each epoch accuracy (x axis) and MCR (y axis) are reported. }
    \label{fig:fashion-mnist_learning}
\end{figure}

We can get a closer look at the trajectory of two networks, which are reported in fig \ref{fig:local_similarity} to better illustrate this phenomenon. Despite their starting point being different, they end up on a minimum with the same MCR/accuracy tradeoff. It seems that the position on the Pareto front only depends on the value of $\tau_{\text{final}}$. Not only the functions are similar at a global scale, but it is also valid at a local scale: while the two nets are only 77\% accurate, they agree on 93\% of the validation samples (ie. they make the same error on the same sample).  

\begin{wrapfigure}{r}{0.55\textwidth}
    \centering
    \begin{tabular}{|c|c|c|}
        \hline
        model & model1 & model2 \\
        \hline
        \hline
        $\tau_i$ & 1.0 & 4.0 \\
        \hline
        $\tau_f$ & 2.0 & 2.0 \\
        \hline
        accuracy (train) & 0.7838 & 0.7793 \\
        \hline
        coincidence (train) & \multicolumn{2}{c|}{0.9421} \\
        \hline
        accuracy (val) & 0.7754 & 0.7724 \\
        \hline
        coincidence (val) & \multicolumn{2}{c|}{0.9350} \\
        \hline
        coincidence (random) & \multicolumn{2}{c|}{0.9937} \\
        \hline
    \end{tabular}
    
    \caption{two \LipCl networks with different initialization and learning curriculum learns similar function as long as the $\tau$ is the same at the end. Although these models only have 69\% accuracy, their predictions match on 92\% of the test set samples.}
    \label{fig:local_similarity}
\end{wrapfigure}
  
Observe that all dots tend to accumulate on the Pareto front, even though they are 12 different networks being trained. It suggests that this method is stable with respect to the input seed. Some of the networks trained with high $\tau_{\text{start}}$ (for example $\tau_{\text{start}}=8$ and $\tau_{\text{final}}=16$) seem to ``lag behind'': empirically we observe that more epochs are required to make the network converge. Hence, the speed at which $\tau$ is modified must be scaled appropriately to ensure that the best Pareto front is recovered.  

This experiment also suggests that a curriculum can be a satisfying approach to tune $\tau$ : an expensive grid search over $\tau$ could be replaced by a single training with a scheduler placed on $\tau$. 
  
\begin{figure}
    \centering
    \includegraphics[scale=0.44]{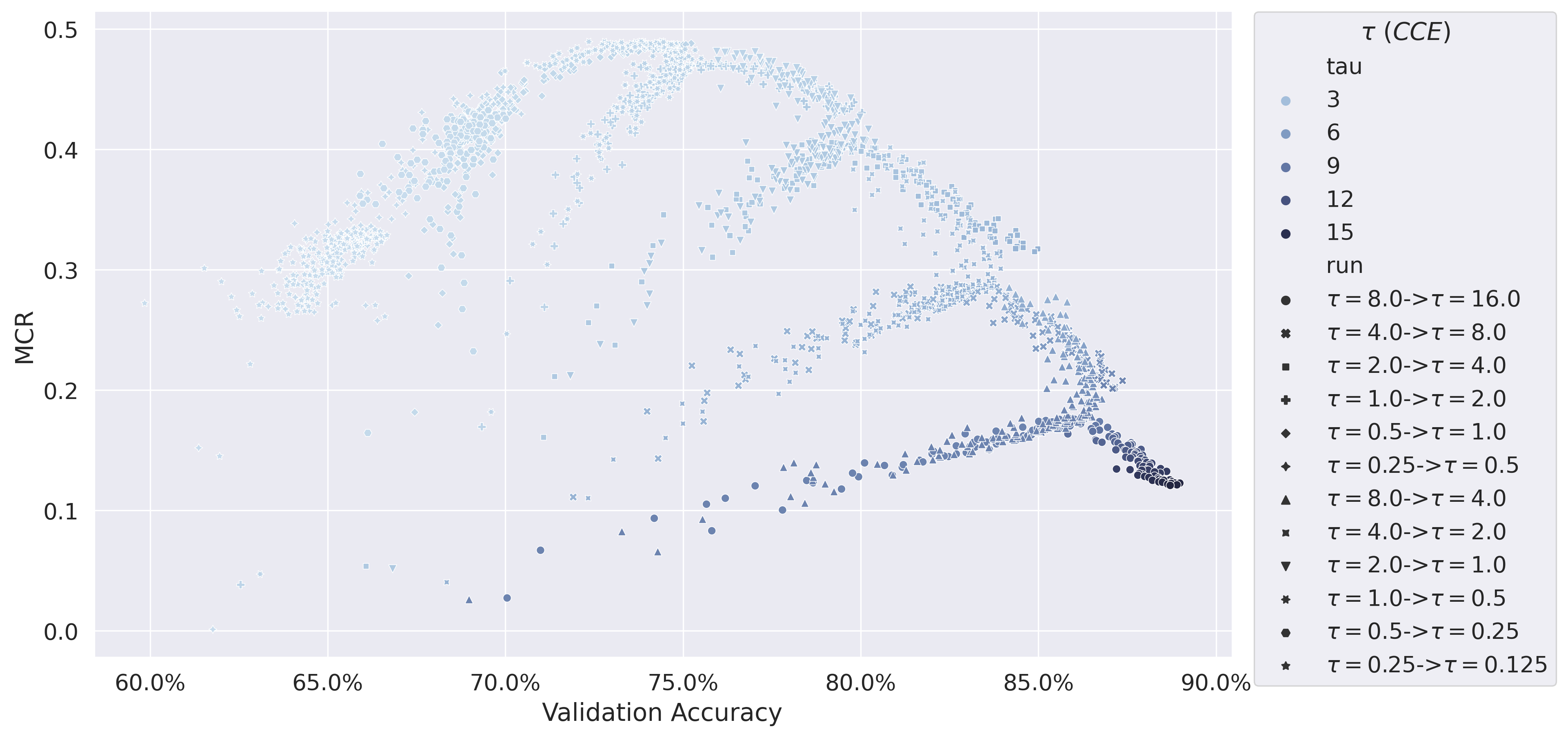}
    \caption{\textbf{MCR/accuracy tradeoff on the validation set of Fashion-Mnist}: 12 \LipCl networks are trained using the schedule depicted in fig \ref{fig:schedule} with respectively $\tau_{\text{init}}\in [0.25, 0.5, 1, 2, 4, 8]$. At each epoch accuracy (x axis) and MCR (y axis) are reported. The Pareto front is still apparent: the MCR/accuracy tradeoff only depends on $\tau$ and not on the initialization.}
    \label{fig:tau_curriculum}
\end{figure}  
  
\subsection{Shifting Pareto front by tuning architecture}

In this experiment we explore an important question: what happens when the architecture is changed? To explore this we perform the same experiment as \ref{app:stability}, but with smaller and larger architectures: the architectures are denoted by the number of filters in their first convolutions (filter of other block are adjusted accordingly by doubling the number of filter of the previous block). It shows how the expressiveness of the architectures affects the Pareto front. We report the results in Figure~\ref{fig:archi_curriculum}.  
  
The validation accuracy can be found on \textbf{x-axis} and the MCR on \textbf{y-axis}. Each dot corresponds to an epoch/a network. Different colors correspond to different architecture widths.
  
As expected, larger networks are more expressive, and as a result, the Pareto front is shifted toward higher accuracy and higher robustness. This observation holds for every scheduling $\tau$. The MCR/accuracy is also architecture dependent.  

\begin{figure}
    \centering
    \includegraphics[scale=0.44]{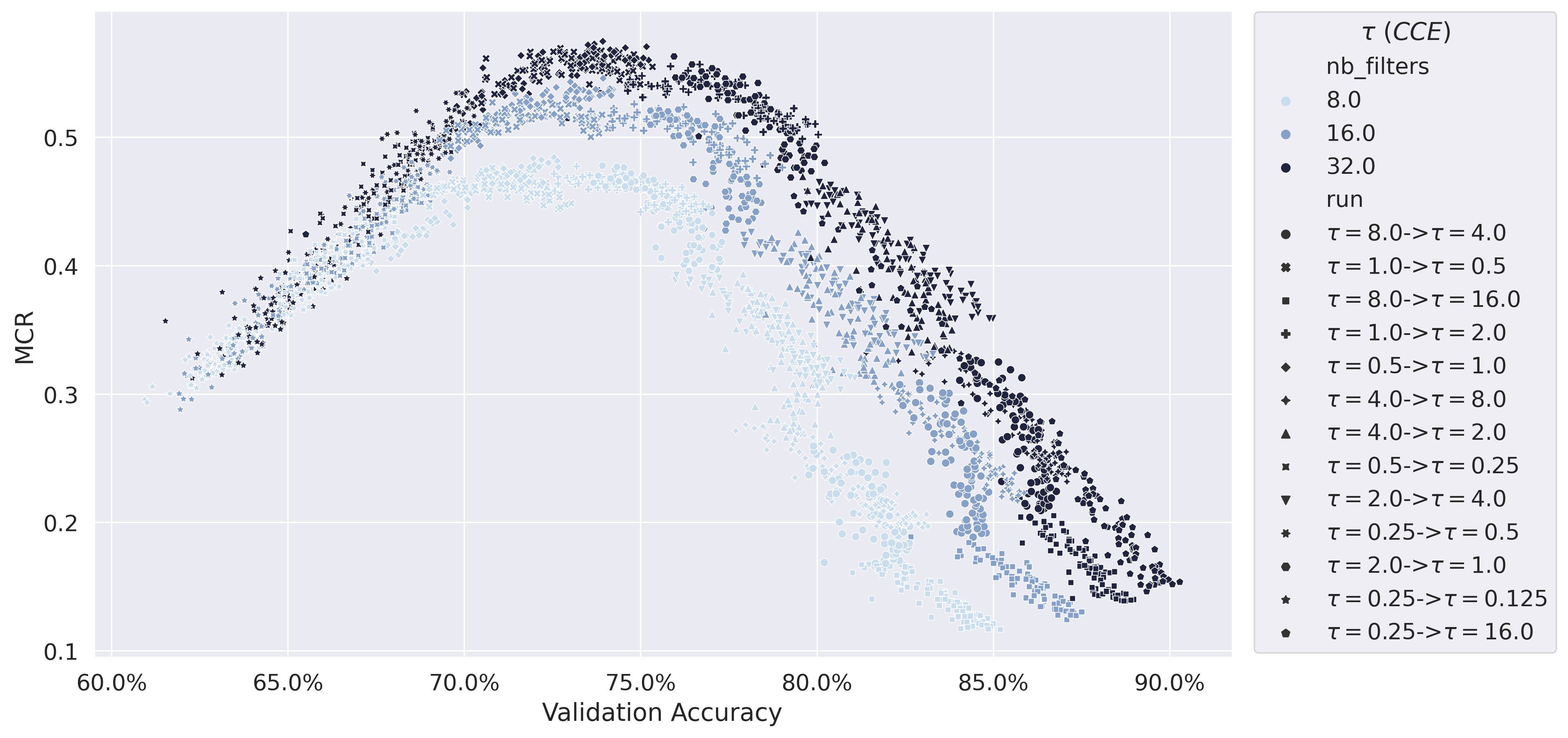}
    \caption{\textbf{MCR/accuracy tradeoff on Fashion-Mnist for \LipCl when architecture size changes:} At each epoch, validation accuracy and MCR are reported. The networks are trained following a scheduling for $\tau$ as described in Figure~\ref{fig:schedule}. We see that larger networks are more expressive, and the Pareto front is shifted toward higher accuracy and higher robustness.}
    \label{fig:archi_curriculum}
\end{figure}

\end{document}